%% file: main.tex
\documentclass{article}

\usepackage[numbers]{natbib}
\usepackage[top=20truemm,bottom=20truemm,left=25truemm,right=25truemm]{geometry}

\usepackage[utf8]{inputenc}
\usepackage{hyperref}
\usepackage{url}
\usepackage{amsfonts}
\usepackage{nicefrac}
\usepackage{microtype}
\usepackage{xcolor}

\title{Non-asymptotic implicit bias of logistic regression \\ at early-stage gradient descent dynamics}

\input{preamble}

\allowdisplaybreaks[4]

\author{%
  Han Bao \\
  The Institute of Statistical Mathematics \\ Tohoku University \\ RIKEN AIP \\
  \texttt{bao.han@ism.ac.jp}
}

\begin{document}

\maketitle

\begin{abstract}%
  Gradient descent has been of particular interest in modern machine learning beyond sole focus on optimization.
  Implicit bias emerging from optimization, though not being encoded by the learning objective, often prevents from overfitting to spurious patterns.
  A typical instance is the max-margin implicit bias of a linear classifier, widely established for exponentially tailed loss functions.
  Even after having a given dataset separated, the parameter vector continues to evolve towards the max-margin direction asymptotically along the gradient descent dynamics.
  This phenomenon corroborates a frequent empirical observation of ``train longer, generalize better.''
  However, the max-margin convergence is an asymptotic phenomenon, and what is worse, this asymptotic convergence rate is significantly slower than pure convex optimization.
  Even so, the parameter vector along gradient descent dynamics commonly correlates with the max-margin direction positively (though not exactly) within considerably fewer iterations than the asymptotic rate.
  By shedding another light on this classical problem, this work aims to understand the mechanism of this early-stage alignment phenomenon.
  Our theoretical results demonstrate that the parameter vector weakly aligns with the max-margin direction within $O(\exp(\exp(-\delta)))$ iterations, where $\delta>0$ is the permissible alignment error,
  which is shown to be tight.
  By tracking the radial and tangential flows, our proof operates on the alignment dynamics directly with dataset geometry and gets rid of the asymptotic expansion, which is a key insight to establishing faster weak alignment.
\end{abstract}

\section{Introduction}
We consider gradient descent with a binary linear classifier $f_\wbf\colon\xbf\mapsto\inpr{\wbf}{\xbf}$, defined by a parameter vector $\wbf\in\Rbb^d$.
Let $\ell:\Rbb\to\Rbb_{\ge0}$ be a convex nonincreasing loss function, typically the logistic or exponential losses, measuring the margin at a given input.
By introducing the training risk $L(\wbf)\defeq n^{-1}\sum_{i=1}^n\ell(\inpr{\wbf}{y_i\xbf_i})$ over a training dataset $\Scal\defeq\{(\xbf_i,y_i)\}_{i=1}^n$, gradient descent recursively operates by
\[
  \wbf({t+1}) \defeq \wbf(t) - \eta\nabla L(\wbf(t)), \quad \text{for $t=0,1,2,\dots$,}
\]
where $\eta>0$ is a fixed stepsize.
During the last decade, it has been shown that gradient descent asymptotically leads the parameter direction $\wbf(t)/\|\wbf(t)\|$ to the max-margin direction $\ubf_*\in\Sbb^{d-1}$ of the dataset $\Scal$.
This is known as the \emph{implicit bias} of gradient descent, which means that the max-margin convergence takes place even though the risk minimization problem does not apparently encode it.
The max-margin implicit bias has been proven by \citet{Soudry2018} under the linear separability condition, and extended to non-separable scenarios~\citep{Ji2019COLT}.
The max-margin convergence has implications on test performance through margin-based generalization bounds~\citep{Koltchinskii2002AOS,Shamir2021JMLR,Schliserman2022COLT} and robustness to input perturbation~\citep{Xu2009JMLR}, rendering it practically relevant.
However, the asymptotic alignment is notoriously slow, at a rate of $\tilde O(1/\log^2t)$, which is unsatisfactory when compared with the standard $O(1/t)$ optimization convergence rate of gradient descent on convex smooth functions.
This slow parameter alignment cannot be improved unless an aggressively large stepsize is used~\citep{Nacson2019AISTATS}, approaching normalized gradient descent.
Yet, can we expect faster parameter alignment for gradient descent in any sense?

\begin{figure}
  \centering
  \includegraphics[height=118pt]{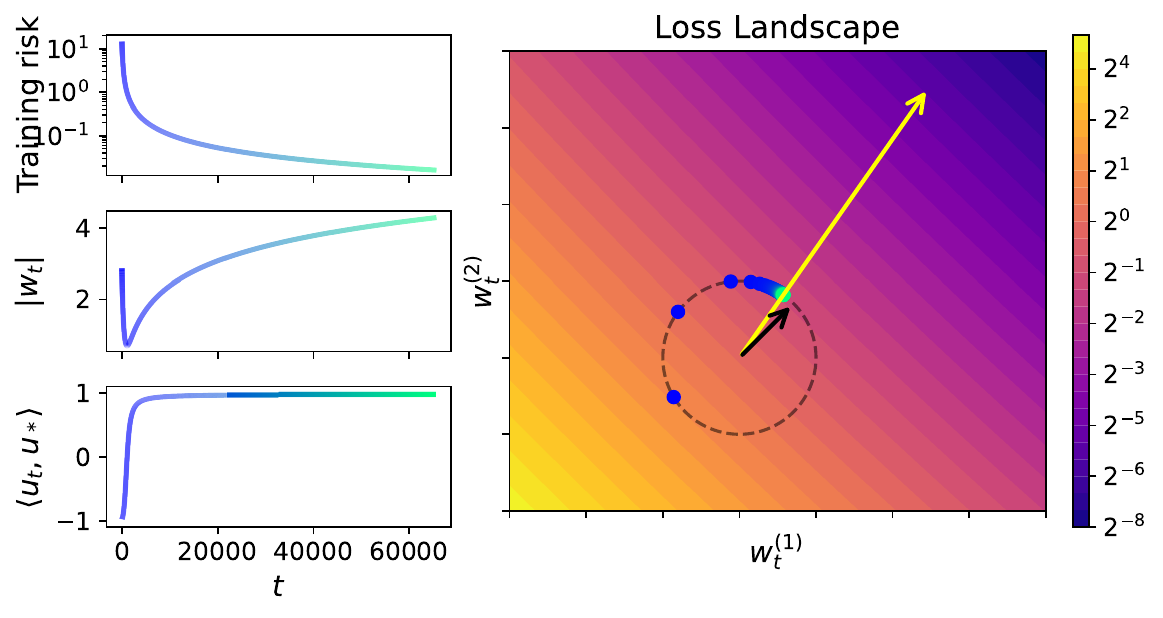}
  \includegraphics[height=120pt]{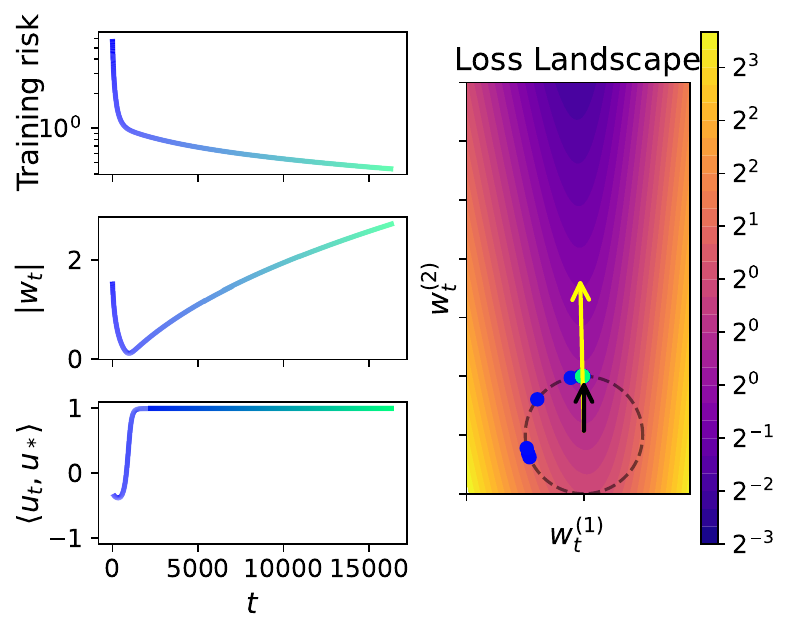}
  \caption{
    Simulation of gradient descent running with the exponential loss $\ell_\text{exp}(m)\defeq\exp(-m)$: \textbf{(Left)} a training dataset with larger margin; \textbf{(Right)} a training dataset with smaller margin.
    Each shows time evolution of the training risk $L(\wbf(t))$, the radial $\|\wbf(t)\|$, the parameter alignment $\inpr{\ubf(t)}{\ubf_*}$ (where $\ubf(t)\defeq\wbf(t)/\|\wbf(t)\|$), and the tangential trajectory (i.e., the trajectory of the normalized parameter $\ubf(t)$) in the loss landscape.
    In the loss landscapes, the dots are the tangential component $\ubf(t)$, and darker and lighter colors indicate earlier and later $t$, respectively;
    $\ubf(t)$ are plotted every $70$ and $50$ iteration for the left and right cases, respectively;
    the black and yellow arrows indicate the max-margin direction $\ubf_*$ and the last parameter $\wbf(T)$, respectively.
    In Appendix~\ref{section:simulation}, the unnormalized parameter trajectories $\wbf(t)$ are additionally shown, as well as the quantitative alignment time.
    The detailed simulation setup is shown there, too.
    \textbf{Key observation:} compared with the radial growth $\|\wbf(t)\|$, the parameter alignment $\inpr{\ubf(t)}{\ubf_*}$ saturates significantly faster, which can also be shown in the parameter trajectory since the tangential trajectory is swept within the early stage, depicted by the darker dots.
  }
  \label{fig:simulation}
\end{figure}

Consider the numerical simulation in Figure~\ref{fig:simulation}, where gradient descent is run for two distinct training datasets, one of which has a large margin (the left figure) and the other has a smaller margin (the right figure).
Denote $\ubf(t)\defeq\wbf(t)/\|\wbf(t)\|$ the tangential component of $\wbf(t)$.
In either case, the radial component $\|\wbf(t)\|$ exhibits slow growth, yet the parameter alignment $\inpr{\ubf(t)}{\ubf_*}$ grows much more rapidly.
This is additionally demonstrated by both parameter trajectories, where the tangential component $\ubf(t)$ is attracted to the max-margin direction $\ubf_*$ (the black arrow) within a remarkably short period, shown by the darker blue dots.
Readers may also refer to the unnormalized parameter trajectories shown in Appendix~\ref{section:simulation} to clearly see this transient behavior.
Whereas the theoretical alignment rate is $\inpr{\ubf(t)}{\ubf_*}=1-\tilde O(1/\log^2 t)$~\citep{Soudry2018}, this asymptotic analysis may not capture well the early-stage behavior.

The aim of this paper is to characterize this early-stage parameter dynamics, bringing forth a deeper understanding of the gradient descent path.
Specifically, our main result rigorously proves that it only takes $O(\exp(\exp(-\delta)))$ time to achieve weak alignment $\inpr{\ubf(t)}{\ubf_*}\ge1-\delta$.
This fast convergence is guaranteed only for moderately but not arbitrarily small $\delta=O(1)$---typically up to $\delta\sim1-\sqrt\gamma$ depending on the dataset margin $\gamma$, which is slightly and non-trivially better than the classical perceptron guarantee $\delta=1-\gamma$~\citep{Novikoff1962}.
In this sense, we call this \emph{weak} alignment, contrasted with \emph{perfect} alignment to guarantee the asymptotic closeness of $\ubf(t)$ to the max-margin direction $\ubf_*$.
We argue that weak alignment is the right notion and necessary compromise to characterize the fast early-stage dynamics observed in Figure~\ref{fig:simulation},
by showing the doubly-exponential time is tight.

\subsection{Setup}
\label{section:setup}
We write $a_t\lesssim b_t$ if there exists an absolute constant $C>0$ such that $a_t\leq Cb_t$ for all $t$.
In this paper, $\|\cdot\|$ denotes the $\ell_2$-norm unless otherwise noted.
The unit hypersphere embedded in $\Rbb^d$ is written as $\Sbb^{d-1}$.
The identity matrix is written as $I$.
The probability simplex is written as $\Delta^{n-1} \defeq \{\pbf\in\Rbb^n \mid \inpr{\pbf}{\onebf}=1\}$, where $\onebf\defeq[1,1,\dots,1]^\top\in\Rbb^n$.
Define $P_\vbf^\perp\in\Rbb^{d\times d}$ the projection operator onto the orthogonal complement to the subspace $\Span(\vbf)\subset\Rbb^d$,
namely, $P_\vbf^\perp\defeq I-\vbf\vbf^\top$.
Similarly, $P_\vbf \defeq \vbf\vbf^\top$ denotes the subspace projection onto $\Span(\vbf)$.
The training dataset $\Scal\defeq\{(\xbf_i,y_i)\}_{i=1}^n \subset \Rbb^d\times\{\pm1\}$ satisfies $\|\xbf_i\|\le1$ for each $i\in[n]$.
We consider linear classifiers trained with the unregularized empirical risk minimization:
\[
  \min_{\wbf\in\Rbb^d} L(\wbf) \defeq \frac1n\sum_{i=1}^n\ell(\inpr{\wbf}{\zbf_i}),
\]
where $\zbf_i\defeq y_i\xbf_i$ (satisfying $\|\zbf_i\|\le1$) and $\ell:\Rbb\to\Rbb_{\ge0}$ is a loss function.
The data points are represented by a matrix $Z\in\Rbb^{d\times n}$ storing each $\zbf_i$ in the column-wise order.
Throughout this paper, we focus on the exponential loss $\ell(m)\defeq e^{-m}$ to simplify the analysis;
yet we believe that the extension to exponentially-tailed loss functions is possible with more algebra.
The training risk is minimized via gradient descent $\wbf(t+1)=\wbf(t) - \eta\nabla L(\wbf(t))$ with a fixed stepsize $\eta>0$ throughout training.
If we fix $t$ and there is no confusion for a moment, we occasionally drop the $t$-dependency for notational simplicity.
For the dataset $\Scal$, let us consider the max-margin problem:
\begin{equation}
  \label{equation:max_margin}
  \min_{\wbf\in\Rbb^d} \frac12\|\wbf\|^2 \text{~~~s.t.~~~}
  \inpr{\wbf}{\zbf_i} \ge 1 \text{~~~for $i\in[n]$},
\end{equation}
and denote $\wbf_*\in\Rbb^d$ the max-margin solution to \eqref{equation:max_margin} and $\ubf_*\defeq\wbf_*/\|\wbf_*\|\in\Sbb^{d-1}$ the max-margin direction.
Correspondingly, we decompose the parameter $\wbf(t)=r(t)\cdot\ubf(t)$ into the radial component $r(t)\defeq\|\wbf(t)\|$ and tangential component $\ubf(t)\defeq\wbf(t)/r(t)\in\Sbb^{d-1}$.
We focus on the linearly separable case throughout the paper,
where the max-margin problem \eqref{equation:max_margin} has a feasible point and $\ubf_*$ gives its direction.
\begin{assumption}
  \label{assump:linearly_separable}
  The dataset $\Scal$ is linearly separable with a positive margin $\gamma>0$,
  namely, we have $\inpr{\ubf_*}{\zbf_i}\ge\gamma$ for all $i\in[n]$.
\end{assumption}
Under the above setup, we focus on the tangential alignment:
$V(t)\defeq1-\inpr{\ubf(t)}{\ubf_*}\in[0,2]$.

\subsection{Related work}
In the era of overparametrized models, classical statistical generalization is too pessimistic~\citep{Zhang2017ICLR},
and optimization and loss landscapes lead to better generalization by stabilizing learning trajectories in a long run~\citep{Hoffer2017NeurIPS,Lv2017ICML},
which inspires the study of implicit bias, questioning which solution an optimizer eventually converges to.
Classically, \citet{Soudry2018} reveals under the same setup as Section~\ref{section:setup} that gradient descent dynamics asymptotically leads the linear parameter to the max-margin direction induced by the $\ell_2$-norm, which is a prototypical expression of implicit bias.
Indeed, a similar max-margin implicit bias is classical for the regularization path~\citep{Rosset2003NeurIPS}.
Yet the convergence rate $V(t)=\tilde O(1/\log^2t)$ is shown to be tight, fundamentally limiting the max-margin implicit bias as an asymptotic behavior.
This slow behavior motivates our work, and we review it in detail in Section~\ref{section:review}.
The max-margin implicit bias has been extended to non-separable datasets~\citep{Ji2019COLT}, convolutional networks~\citep{Gunasekar2018NeurIPS}, general loss functions~\citep{Gunasekar2018ICML,Nacson2019AISTATS,Ji2020COLT,Ravi2024NeurIPS}, homogeneous networks~\citep{Ji2020ICLR,Lyu2020ICLR}, and transformers~\citep{Tarzanagh2023NeurIPS}.
In contrast, implicit bias has been generalized for other optimizers, including Adam~\citep{Zhang2024NeurIPS}, AdamW~\citep{Xie2024ICML}, steepest descent~\citep{Tsilivis2025ICLR}, mirror descent~\citep{Liang2025ICLR}, and spectral descent~\citep{Fan2025NeurIPS},
which typically leads to the max-margin direction with a different norm from $\|\cdot\|_2$.
In functional spaces, we can view implicit bias as linear spline interpolation from the variational perspective~\citep{Savarese2019COLT,Ongie2020ICLR,Ardeshir2023COLT}.
Moreover, the implicit bias can be used to describe feature learning~\citep{Woodworth2020COLT,Vasudeva2025NeurIPS}, where different initialization leads to implicit regularization enhancing feature selection.

The aforementioned studies on implicit bias are more or less ``static'' because their focus is an asymptotic solution of an optimizer.
To better control the training schedule, understanding of optimizer's dynamical behaviors matters.
In feature-learning scenarios, two-stage dynamics are often observed:
parameters seek signal directions in the first stage, and then the parameter norms grow to decrease the loss in the second stage~\citep{Cao2022NeurIPS,Boursier2022NeurIPS,Glasgow2024ICLR}.
This line of papers carefully characterizes the second-stage behavior to establish benign overfitting for two-layer neural networks; nevertheless, training points are already perfectly classified in the initial alignment stage.
Less is known for linear classifiers, while we empirically observe a clear early-stage behavior.
Recently, gradient descent dynamics of linear classifiers has been revisited with a large stepsize, where we observe a two-stage behavior even in the linear case~\citep{Wu2024COLT,Meng2024,Tyurin2025AAAI,Bao2025NeurIPS}.
Therein the dynamics exhibits a two-stage behavior, the initial oscillation stage seeking an approximate max-margin direction, and the late stabilization stage to converge asymptotically.
Our analysis is partially inspired by the two-stage behaviors of linear models with a large stepsize.
More discussion is deferred to Section~\ref{section:discussion}.

\subsection{Review of asymptotic gradient descent dynamics}
\label{section:review}
We review the implicit bias result by \citet{Soudry2018} in this section to describe how the gradient descent dynamics exhibits an asymptotically slow motion.
Consider the setup as in Section~\ref{section:setup} with the linear separability (Assumption~\ref{assump:linearly_separable}).
\begin{theorem}[{\citet{Soudry2018}}]
  \label{theorem:soudry}
  For any stepsize $\eta<1/L(\wbf(0))$ and initialization $\wbf(0)$, the tangential parameter component $\ubf(t)$ following the gradient descent dynamics converges to the max-margin direction $\ubf_*$ defined via \eqref{equation:max_margin} asymptotically.
  The asymptotic convergence rates are given as follows:
  \[
    \begin{cases}
      \text{Risk convergence:}       & L(\wbf(t))=O(1/t); \\
      \text{Radial divergence:}      & r(t)=\Theta(\log t); \\
      \text{Tangential convergence:} & V(t)=\tilde O(1/\log^2t).
    \end{cases}
  \]
\end{theorem}
Whereas the risk converges reasonably in $O(1/t)$ as in typical convex smooth (but not strongly convex) optimization,
the tangential component $\ubf(t)$ converges to the max-margin direction $\ubf_*$ very slowly.
The polylogarithmic rate is translated into the asymptotic tangential convergence time $t=T_\infty=\tilde O(\exp(\delta^{-1/2}))$.

Here we briefly illustrate the proof sketch of the convergence rates to highlight obstacles that we need to overcome.
The risk convergence and radial divergence are easy to see from our structural results in Lemma~\ref{lemma:l_r_bounds}, which we show later in Section~\ref{section:structure}.
Suppose that we have proven them and focus on the tangential convergence.
Given the asymptotic convergence $\ubf(t)\to\ubf_*$ and the radial divergence rate,
we can write $\wbf(t)=\ubf_*\log t+\xibf(t)$ with a vanishing remainder $\|\xibf(t)\|=o(\log t)$.
Then the tangential component $\ubf=\wbf/\|\wbf\|$ admits the following asymptotic expansion:
\[
  \begin{aligned}
    \ubf = \frac{\wbf}{\|\wbf\|}
    &= \frac{\ubf_* + \xibf\cdot(\log t)^{-1}}{\sqrt{1 + 2\inpr{\ubf_*}{\xibf}\cdot(\log t)^{-1} + \|\xibf\|^2(\log t)^{-2}}} \\
    &= \left(\ubf_*+\frac{\xibf}{\log t}\right)\left[1 - \frac{\inpr{\ubf_*}{\xibf}}{\log t} + \frac{3\inpr{\ubf_*}{\xibf}^2-\|\xibf\|^2}{2\log^2t} - O\left(\frac{1}{\log^3t}\right)\right] \\
    &= \ubf_* + P_{\ubf_*}^\perp\xibf\cdot(\log t)^{-1} - O(1/\log^2t),
  \end{aligned}
\]
where we used the Taylor expansion $\frac{1}{\sqrt{1+x}}=1-\frac12x+\frac38x^2+O(x^3)$.
By noting $\langle{\ubf_*},{P_{\ubf_*}^\perp\xibf}\rangle=0$, we have $\inpr{\ubf}{\ubf_*}=1-O(1/\log^2t)$,
where an additional polylogarithmic factor may arise under a pathological case.
Thus, we verify the tangential convergence rate $V(t)=\tilde O(1/\log^2t)$.

While being tight in large $t$, a caveat of this analysis is the application of the asymptotic expansion,
and thus non-asymptotic factors are hidden in the remainder $O(1/\log^2t)$.
Our subsequent analysis directly operates on the tangential alignment $V(t)$, instead of asymptotically expanding it.

\section{Structural properties on radial and tangential flows}
\label{section:structure}

First, we consider the gradient flow $\dot\wbf(t) = -\nabla L(\wbf(t))$
to focus on the fundamental structure of the dynamics,
where we use the same $t$ to denote time instead of discrete steps with an abuse of notation.
The dynamics decomposition (Lemma~\ref{lemma:decomposition}) and radial bounds (Lemma~\ref{lemma:l_r_bounds}) are not technically novel, but we show their concise proofs to let readers better understand the preliminary.
The discrete-time dynamics is analyzed later in Section~\ref{section:discrete} by controlling an extra discretization error via the stepsize choice.
We initialize the parameter with $\|\wbf(0)\|=\rho$ for some $\rho>\gamma^{-2}\log L(0)$.%
\footnote{
  This condition on $\rho$ ensures that the polar decomposition exists:
  $\ubf(0)$ is well-defined, and $r(t)>0$ holds across $t>0$, which is indirectly shown in the proof of Theorem~\ref{lemma:lyapunov_lower_bound}---%
  the radial flow satisfies $r(t)>\underline{r}=r(T_0)$ therein.
}
Let us decompose the parameter dynamics into the radial and tangential components.
\begin{lemma}[Radial/tangential flows]
  \label{lemma:decomposition}
  For $t\ge 0$, the gradient flow $\dot\wbf(t) = -\nabla L(\wbf(t))$ can be decomposed as follows:
  \begin{equation}
    \label{equation:decomposition}
    \dot r(t) = -\inpr{\ubf(t)}{\nabla L(\wbf(t))}, \qquad
    \dot\ubf(t) = -\frac{1}{r(t)}P_{\ubf(t)}^\perp\nabla L(\wbf(t)).
  \end{equation}
\end{lemma}
\begin{proof}
  We have by the product rule that
  $\dot\wbf = \dot r \ubf + r \dot\ubf$.
  Taking the inner product with $\ubf$ on both sides and using $\|\ubf\|=1$ and $\inpr{\ubf}{\dot\ubf}=0$ (which follows from differentiating $\|\ubf\|^2=1$), we obtain
  \[
    \inpr{\ubf}{\dot\wbf} = \dot r \inpr{\ubf}{\ubf} + r \inpr{\ubf}{\dot\ubf} = \dot r.
  \]
  Since $\dot\wbf = -\nabla L(\wbf)$, this yields $\dot r = -\inpr{\ubf}{\nabla L(\wbf)}$.

  For the tangential flow, we rearrange the product rule to get
  \[
    \dot\ubf = \frac{1}{r}\bigl(\dot\wbf - \dot r \ubf\bigr)
    = \frac1r\bigl[-\nabla L(\wbf) + \inpr{\ubf}{\nabla L(\wbf)} \ubf\bigr]
    = -\frac1r\bigl(I-\ubf\ubf^\top\bigr)\nabla L(\wbf)
    = -\frac1rP_{\ubf}^\perp\nabla L(\wbf),
  \]
  where we used $\ubf\ubf^\top \nabla L(\wbf) = \inpr{\ubf}{\nabla L(\wbf)}\ubf$.
\end{proof}
Next, we provide upper and lower bounds on $L(\wbf(t))$ and $r(t)$.
Hereafter, we frequently use the exponential weighting $\alphabf(t)\in\Delta^{n-1}$,
where each data point $i\in[n]$ is weighted by $\alpha_i(t)\defeq e^{-\inpr{\wbf(t)}{\zbf_i}}/\sum_je^{-\inpr{\wbf(t)}{\zbf_j}}$.
\begin{lemma}[Risk/radial bounds]
  \label{lemma:l_r_bounds}
  For $t\ge0$, we have the following bounds:
  \[
    \begin{cases}
      \text{Risk estimates:}
      & \frac{1}{L(0)^{-1}+t} \le L(\wbf(t)) \le \frac{1}{L(0)^{-1}+\gamma^2 t}, \\
      \text{Radial estimates:}
      & \rho-\frac{1}{\gamma^2}\log(1+L(0)\gamma^2t) \le r(t) \le \rho+\frac{1}{\gamma^2}\log(1+L(0)\gamma^2t), \\
      \text{Radial estimates (refined):}
      & \log(L(0)^{-1}+\gamma^2 t) \le r(t).
    \end{cases}
  \]
\end{lemma}
\begin{proof}
  We use $L(t)$ as the shorthand of $L(\wbf(t))$.
  First, we derive the upper bound of $L(t)$.
  By expanding $\|\nabla L(t)\|$, we have
  \begin{equation}
    \label{equation:self_bounding}
    \|\nabla L(t)\|
    = \frac1n\left\|\sum_{i=1}^ne^{-\inpr{\wbf(t)}{\zbf_i}}\zbf_i\right\|
    = L(t)\left\|\sum_{i=1}^n\alpha_i(t)\zbf_i\right\|
    = L(t)\|Z\alphabf(t)\| \in [\gamma L(t), L(t)],
  \end{equation}
  where $\gamma\le|\inpr{\ubf_*}{Z\alphabf}|\le\|Z\alphabf\|\le1$ is used.
  We often call the inequality $\gamma L\le\|\nabla L\|\le L$ in~\eqref{equation:self_bounding} the \emph{self-bounding property} of a loss function.
  Readers should keep this in mind since we repeatedly use it.
  To derive the bounds of $L(t)$, we evaluate its time derivative
  \[
    \dot L(t)
    = -\frac1n\sum_{i=1}^ne^{-\inpr{\wbf(t)}{\zbf_i}}\inpr{\dot\wbf(t)}{\zbf_i}
    = -\|\nabla L(t)\|^2,
  \]
  which together with \eqref{equation:self_bounding} yields the differential inequality $-L(t)^2\le\dot L(t)\le-\gamma^2L(t)^2$.
  After dividing all by $L(t)^2$ and integrating from $0$ to $t$, we obtain the bounds of $L(t)$.

  To derive the bounds of $r(t)$, we expand the radial flow in \eqref{equation:decomposition} as follows:
  \[
    \dot r(t)
    = -\inpr{\ubf(t)}{\nabla L(t)}
    = \frac1n\sum_{i=1}^ne^{-\inpr{\wbf(t)}{\zbf_i}}\inpr{\ubf(t)}{\zbf_i}
    = L(t)\inpr{\ubf(t)}{Z\alphabf(t)}.
  \]
  Then we have the differential inequality $-L(t)\le\dot r(t)\le L(t)$,
  where the Cauchy--Schwarz inequality $|\inpr{\ubf}{Z\alphabf}|\le\|Z\alphabf\|\le1$ is used.
  By plugging $L(t)\le\frac{1}{L(0)^{-1}+\gamma^2t}$ and integrating from $0$ to $t$, we obtain the bounds of $r(t)$.

  Finally, we derive the last lower bound of $r(t)$.
  Noting $\inpr{\wbf(t)}{\zbf_i}=r(t)\inpr{\ubf(t)}{\zbf_i}\le r(t)$, we have $L(t) \ge e^{-r(t)}$.
  Combining with $L(t)\le\frac{1}{L(0)^{-1}+\gamma^2t}$, we have $r(t)\ge\log(L(0)^{-1}+\gamma^2t)$.
\end{proof}
The radial estimate $r(t)\ge \rho-\frac{1}{\gamma^2}\log(1+L(0)\gamma^2t)$ is vacuous with large $t$, and the refined estimate $r(t)\ge\log(L(0)^{-1}+\gamma^2t)$ is then effective.
This is particularly used in the lower bound (Theorem~\ref{lemma:lyapunov_lower_bound}).

\section{Main result}
\label{section:main}
To analyze the tangential alignment non-asymptotically, let us derive the differential inequality:
\begin{equation}
  \label{equation:dot_v}
  \dot V(t)
  = \frac{1}{r(t)}\inpr{\nabla L(\wbf(t))}{P_{\ubf(t)}^\perp\ubf_*}
  = -\frac{L(\wbf(t))}{r(t)}\inpr{Z\alphabf(t)}{P_{\ubf(t)}^\perp\ubf_*},
\end{equation}
where we applied the tangential flow in~\eqref{equation:decomposition} to $V(t)=1-\inpr{\ubf(t)}{\ubf_*}$.
On the right-hand side of~\eqref{equation:dot_v}, $L(\wbf(t))$ and $r(t)$ can be easily handled by using Lemma~\ref{lemma:l_r_bounds},
yet the term $\inpr{Z\alphabf}{P_{\ubf}^\perp\ubf_*}$ requires more convoluted treatment.
Nonetheless, \emph{we can make an ansatz $\inpr{Z\alphabf}{P_{\ubf}^\perp\ubf_*}>\kappa$ for some $\kappa>0$}, which we rigorously justify later in Lemma~\ref{lemma:uniform_bound}.
The geometric intuition behind this ansatz is as follows---%
the exponentially-weighted data $Z\alphabf$ should positively correlate with $\ubf_*$.
This correlation persists even with $P_{\ubf}^\perp\ubf_*$, the projected $\ubf_*$ onto $\Span(\ubf)^\perp$, unless $\ubf$ and $\ubf_*$ are nearly antipodal.
With this ansatz, the non-asymptotic alignment is relatively easy to establish.
\begin{lemma}[Weak alignment under ansatz]
  \label{lemma:alignment}
  For an absolute constant $\kappa>0$, assume that the following uniform bound is valid for $t\in[0,T_1]$:
  \begin{equation}
    \label{equation:lyapunov_uniform_bound}
    \inpr{Z\alphabf(t)}{P_{\ubf(t)}^\perp\ubf_*} \ge \kappa.
  \end{equation}
  Then, for $t\in[0,T_1]$, we have
  \begin{equation}
    \label{equation:v_ub}
    V(t) \le
    V(0) - \frac{\kappa\gamma^2}{(1+\gamma^2)(\rho\gamma^2+\log 2)} - \frac{\kappa\gamma^2}{2}\log\left(1+\frac{\log(L(0)\gamma^2 t)}{\rho\gamma^2+\log2}\right).
  \end{equation}
  In this case, $V(t)\le\delta$ is attained after $T_1=O(\exp(\exp(-\delta/\kappa\gamma^2)))$ time.
\end{lemma}
\begin{proof}
  Plugging~\eqref{equation:lyapunov_uniform_bound} together with Lemma~\ref{lemma:l_r_bounds} into the differential equation~\eqref{equation:dot_v}, we have
  \[
    \dot V(t) \le -\frac{\kappa\gamma^2}{(L(0)^{-1}+t)[\rho\gamma^2+\log(1+\gamma^2L(0)t)]}
    \quad \text{for $0\le t\le T_1$.}
  \]
  This differential inequality is elementary and can be solved via Lemma~\ref{lemma:integral_polylog} (in Appendix~\ref{section:auxiliary}) with $a\defeq L(0)^{-1}$, $b\defeq\rho\gamma^2$, and $c\defeq\gamma^2L(0)$.
  Then, we have the desired upper bound~\eqref{equation:v_ub}.
  Solving $\text{(upper bound of \eqref{equation:v_ub})}\le\delta$ with some relaxation by dropping the second term (which is negative) leads to
  \[
    t \ge \frac{1}{L(0)\gamma^2}\exp\left\{(\rho\gamma^2+\log2)\left[\exp\left(\frac{2(V(0)-\delta)}{\kappa\gamma^2}\right)-1\right]\right\}.
  \]
  Hence $T_1=O(\exp(\exp(-\delta/\kappa\gamma^2)))$ is sufficient time for achieving weak alignment $V(t)\le\delta$.
\end{proof}

Lemma~\ref{lemma:alignment} already exhibits the non-asymptotic tangential convergence $V(t)\lesssim C-\log\log t$.
Or equivalently, $V(t)\le\delta$ is finitely achieved for $t=O(\exp(\exp(-\delta/\kappa\gamma^2)))$.
This doubly-logarithmic exponential convergence is much faster than the asymptotic convergence rate $\tilde O(1/\log^2t)$ in Theorem~\ref{theorem:soudry}.
The explicit rate depends on the absolute constant $\kappa$.

Subsequently, we need to verify the uniform bound~\eqref{equation:lyapunov_uniform_bound}.
We can establish the following geometric lemma to show a lower bound of form~\eqref{equation:lyapunov_uniform_bound} (but conditioned on $V\le1$).
Note that this uniform bound is purely geometric, so we can readily use it for both continuous-time and discrete-time cases.
\begin{figure}[t]
  \centering
  \begin{minipage}[t]{0.48\textwidth}
    \centering
    \input{figures/geometry_diagram}
  \end{minipage}
  \hfill
  \begin{minipage}[t]{0.48\textwidth}
    \centering
    \input{figures/geometry_diagram2}
  \end{minipage}
  \caption{
    The geometry behind the proof of Lemma~\ref{lemma:uniform_bound}, $\inpr{Z\alphabf}{P_{\ubf}^\perp\ubf_*}\ge\kappa_0(V)$.
    \textbf{(Left)} In the first step, we decompose $P_\ubf^\perp\ubf_* = \ubf_* - P_\ubf\ubf_*$ and derive bounds for $\inpr{\ubf_*}{Z\alphabf}$ and $\inpr{P_\ubf\ubf_*}{Z\alphabf}$ independently.
    \textbf{(Right)} In the second step, we decompose $\ubf = P_{\ubf_*}\ubf + P_{\ubf_*}^\perp\ubf$ and derive bounds for $\inpr{P_{\ubf_*}\ubf}{\bar\zbf}$ and $\inpr{P_{\ubf_*}^\perp\ubf}{\bar\zbf}$ independently.
  }
  \label{figure:geometry}
\end{figure}
\begin{lemma}[Geometric lemma]
  \label{lemma:uniform_bound}
  For $t\ge0$, if $V\le1$, we have the following bounds:
  \begin{equation}
    \label{equation:uniform_bound}
    \inpr{Z\alphabf}{P_{\ubf}^\perp\ubf_*} \ge
    \gamma - \overbrace{(1-V)[\underbrace{(1-V)\bar\gamma}_{\text{bounds $\inpr{P_{\ubf_*}\ubf}{\bar\zbf}$}}+\underbrace{\sqrt{(1-\bar\gamma^2)(2V-V^2)}}_{\text{bounds $\langle{P_{\ubf_*}^\perp\ubf},{\bar\zbf}\rangle$}}]}^{\text{bound $\inpr{P_\ubf\ubf_*}{Z\alphabf}$}} \quad (\eqdef \kappa_0(V)),
  \end{equation}
  where $\bar\gamma\defeq\inpr{\ubf_*}{\bar\zbf}$ is the margin of the data mean $\bar\zbf\defeq\frac1n\sum_{i=1}^n\zbf_i$.
\end{lemma}
\begin{proof}
  Using $\inpr{\ubf_*}{Z\alphabf}\ge\gamma$, we decompose $P_\ubf^\perp\ubf_* = \ubf_* - P_\ubf\ubf_*$ and bound $\inpr{Z\alphabf}{P_{\ubf}^\perp\ubf_*}$ as follows (see Figure~\ref{figure:geometry}~\textbf{(Left)} for the geometric illustration):
  \[
    \langle Z\alphabf,P_{\ubf}^\perp\ubf_*\rangle
    = \inpr{Z\alphabf}{\ubf_*-\inpr{\ubf}{\ubf_*}\ubf}
    \ge \gamma - (1-V)\inpr{\ubf}{Z\alphabf}
    \ge \gamma - (1-V)\inpr{\ubf}{\bar\zbf},
  \]
  where we used $\inpr{\ubf}{Z\alphabf}\le\inpr{\ubf}{\bar\zbf}$ implied by Lemma~\ref{lemma:temperature_monotonicity} (in Appendix~\ref{section:auxiliary}; this is the temperature monotonicity of the exponentially weighted average) with $m_i=\inpr{\ubf}{\zbf_i}$ and $\beta=r$.
  Then, orthogonally decompose $\ubf=\inpr{\ubf}{\ubf_*}\ubf_*+\ubf_\perp$, where $\ubf_\perp\defeq P_{\ubf_*}^\perp\ubf$.
  The orthogonal component has $\|\ubf_\perp\|^2=1-\inpr{\ubf}{\ubf_*}^2 = 2V-V^2$.
  Using this, we further have $\inpr{\ubf}{\bar\zbf}=(1-V)\bar\gamma+\inpr{\ubf_\perp}{\bar\zbf}$---this decomposition is illustrated in Figure~\ref{figure:geometry}~\textbf{(Right)}.
  Lastly, we bound $\inpr{\ubf_\perp}{\bar\zbf}$ as follows:
  \[
    \begin{aligned}
      \inpr{\ubf_\perp}{\bar\zbf}
      &= \inpr{\ubf_\perp}{P_{\ubf_*}^\perp\bar\zbf}
        && \quad\text{(by $\ubf_\perp \perp \ubf_*$)} \\
      &\le \|\ubf_\perp\| \cdot \|P_{\ubf_*}^\perp\bar\zbf\|
        && \quad\text{(Cauchy--Schwarz)} \\
      &= \|\ubf_\perp\| \cdot \|\bar\zbf - \bar\gamma\ubf_*\|
      = \|\ubf_\perp\|\sqrt{\|\bar\zbf\|^2 - \bar\gamma^2}
        && \quad\text{(by $P_{\ubf_*}^\perp = I-\ubf_*\ubf_*^\top$)} \\
      &\le \|\ubf_\perp\|\sqrt{1-\bar\gamma^2}
      = \sqrt{(1-\bar\gamma^2)(2V-V^2)}.
        && \quad\text{(by $\|\bar\zbf\|\le1$)}
    \end{aligned}
  \]
  By aggregating altogether, we have the lower bound~\eqref{equation:uniform_bound}.
\end{proof}

We alleviate the $t$-dependency of the lower bound \eqref{equation:uniform_bound} (via $V$) later;
with that, Lemma~\ref{lemma:alignment} implies the early-stage weak alignment.
Lemma~\ref{lemma:uniform_bound} requires that the initialization is not too bad: $V(0)\le1$.
If $V(0)>1$ holds, the initial tangential component $\ubf(0)$ and the max-margin direction $\ubf_*$ are not in the same hemisphere;
still, the flow rapidly escapes from such a ``bad'' initialization, as we show below.
\begin{lemma}[Escape from poor initialization]
  \label{lemma:escape}
  Let
  $W(t)\defeq\inpr{\wbf(t)}{\ubf_*}=r(t)(1-V(t))$ be the parameter projection, and
  \(
    T_\text{escape}\coloneqq L(0)^{-1}\left\{\exp(\rho\gamma^{-1}[V(0)-1])-1\right\}
  \)
  .
  If $V(0)>1$, then the flow escapes the bad initialization, namely, $V(t)\le 1$,
  after $t\ge T_\text{escape}$.
\end{lemma}
\begin{proof}
  With $\inpr{\ubf_*}{Z\alphabf}\ge\gamma$ and Lemma~\ref{lemma:l_r_bounds},
  we derive the differential inequality of $W(t)$ as follows:
  \begin{equation}
    \label{equation:dot_w}
    \dot W(t) = \inpr{\dot\wbf(t)}{\ubf_*} = -\inpr{\nabla L(\wbf(t))}{\ubf_*} = L(\wbf(t))\inpr{Z\alphabf(t)}{\ubf_*}
    \ge \frac{\gamma}{L(0)^{-1}+t}.
  \end{equation}
  Integrating over $[0,t]$, we have
  $W(t)\ge W(0)+\gamma\log(1+L(0)t)$.
  Note that $W(t)\in\Rbb$ in general.
  To escape from poor initialization (or to have $V(t)\le1$),
  it is sufficient to have $W(t)=r(t)(1-V(t))\ge0$.
  Hence, by solving $W(0)+\gamma\log(1+L(0)t)\ge0$,
  we have $t\ge T_\text{escape}$.
\end{proof}

We are now ready to show our main result, early-stage tangential alignment, by combining Lemmas~\ref{lemma:alignment},~\ref{lemma:uniform_bound},~and~\ref{lemma:escape}.
Here, the parameter first escape from a bad initialization (Lemma~\ref{lemma:escape}; escape stage), and then weakly aligns with the max-margin direction (Lemma~\ref{lemma:alignment}; weak alignment stage).

\begin{tcolorbox}[highlightbox]
\begin{theorem}[Early-stage weak alignment]
  \label{theorem:early_stage}
  Under the setup in Section~\ref{section:setup}, consider the radial and tangential flows in \eqref{equation:decomposition}.
  There exists a problem-dependent threshold $\bar\delta>0$ such that the following holds:
  pick a fixed constant $\Delta>\bar\delta$, then
  for any $\delta>\Delta$,
  the tangential alignment achieves $V(t)\le\delta$ after $t=O(\exp(\exp(-\delta/\kappa_0(\Delta)\gamma^2)))$ time.
\end{theorem}
\end{tcolorbox}

We compare our early-stage alignment and perfect alignment of Theorem~\ref{theorem:soudry} in Figure~\ref{fig:stages}.
The early-stage alignment first undergoes the escaping stage, taking $T_\text{escape}=O(1)$ time---by regarding initialization and the dataset margin $\gamma$ as problem-dependent constants.
Then it undergoes the weak alignment stage to reach $V(t)\le\delta$.
To apply Lemma~\ref{lemma:alignment}, we need a uniform lower bound $\kappa$ on $\langle{Z\alphabf},{P_{\ubf}^\perp\ubf_*}\rangle$ throughout the stage when $V\in[\delta,1]$.
As we show in Section~\ref{section:proof_main}, for any fixed $\Delta>\bar\delta$, we can take $\kappa=\kappa_0(\Delta)>0$, which is a positive constant independent of $t$ and $\delta$.
With this choice, the weak alignment time is $T_1=O(\exp(\exp(-\delta/\kappa_0(\Delta)\gamma^2)))=O(\exp(\exp(-\delta)))$ for all $\delta>\Delta$, where $\kappa_0(\Delta)$ and $\gamma$ are absorbed into the absolute constant.
Therefore, the early-stage alignment is basically dominated by the weak alignment stage time $T_1$,
and this is substantially faster than the asymptotic perfect alignment $T_\infty=\tilde O(\exp(\delta^{-1/2}))$.

Nevertheless, we must note again that the early-stage alignment can only achieve \emph{weak} alignment, that is, $V(t)\le\delta$ cannot be achieved for an arbitrarily small error $\delta$.
From Theorem~\ref{theorem:early_stage}, we only achieve alignment up to some dataset-dependent $\bar\delta$.
In Section~\ref{section:proof_main}, we will have a tight approximation bound $\bar\delta\le1-\gamma$.
Roughly, this indicates that $\inpr{\ubf}{\ubf_*}\ge\gamma$ is achieved through weak alignment---%
or it can be improved up to $\inpr{\ubf}{\ubf_*}\ge\sqrt{\gamma}$ for a dataset with the data-mean margin $\bar\gamma=\inpr{\ubf_*}{\bar\zbf}\approx1$.
See Appendix~\ref{section:technicality} for details.
This alignment is strictly weaker than the asymptotic tangential convergence, but still numerically non-negligible.\footnote{
  For example, $\sqrt\gamma\ge0.31$ when $\gamma=0.1$ and $\sqrt\gamma\ge0.44$ when $\gamma=0.2$.
  Even with such a small dataset margin $\gamma$, the weak-alignment threshold $\sqrt\gamma$ is non-trivially large and cannot be achieved by merely the perceptron argument~\citep{Novikoff1962}.
  See Table~\ref{tab:alignment_time} in Appendix~\ref{section:simulation} for numerical simulation to show necessary time attaining different $\gamma$ and $\sqrt\gamma$.
}

\subsection{Proof of Theorem~\ref{theorem:early_stage}}
\label{section:proof_main}
The escaping stage only takes $T_\text{escape}=O(1)$ at most, as seen in Lemma~\ref{lemma:escape}.
Hereafter, we assume that our initialization satisfies $V(0)<1$ to establish the weak-alignment stage and ``re-index'' $t$.
Note that $V(t) \le 1$ persists thereafter because $W(t)=r(t)(1-V(t))$ is strictly increasing due to \eqref{equation:dot_w}, which maintains $W(t)\ge 0$ throughout $t$ after the escape.

\begin{figure}[t]
  \centering
  \begin{minipage}[t]{0.54\textwidth}
    \centering
    \input{figures/stages_diagram}
    \captionof{figure}{
      The two-stage dynamics:
      the \textbf{escape} stage quickly moves away from poor initialization,
      followed by the \textbf{weak alignment} stage achieving $\inpr{\ubf}{\ubf_*}\ge1-\delta$.
    }
    \label{fig:stages}
  \end{minipage}
  \hfill
  \begin{minipage}[t]{0.44\textwidth}
    \centering
    \input{figures/kappa_plot}
    \captionof{figure}{
      Plot of $\kappa_0(V)$ for $V\in[0,1]$ with parameters $\gamma=0.8$ and $\bar\gamma=0.9$,
      which is strictly increasing on $\bar\delta<V<1$.
    }
    \label{fig:kappa_plot}
  \end{minipage}
\end{figure}

The weak-alignment stage is proven by concatenating Lemmas~\ref{lemma:alignment}~and~\ref{lemma:uniform_bound},
but we still need to check that $\inpr{Z\alphabf}{P_{\ubf}^\perp\ubf_*}$ has a non-degenerate lower bound to invoke Lemma~\ref{lemma:alignment}
because the lower bound~\eqref{equation:uniform_bound} provided by Lemma~\ref{lemma:uniform_bound} is time-dependent (through $V$), whose strict positivity is not obvious.
Recall
\[
  \kappa_0(V) \defeq \gamma - (1-V)[(1-V)\bar\gamma+\sqrt{(1-\bar\gamma^2)(2V-V^2)}].
\]
Subsequently, we mainly focus on checking $\kappa_0(V)>0$ in the range of our interest.
By rearranging $\kappa_0(V)>0$, we have the quadratic inequality
$U^2-(1+2\gamma\bar\gamma-\bar\gamma^2)U+\gamma^2>0$ with $U\defeq(1-V)^2$.
Solving this, we have $\bar\delta<V<1$, where
\begin{equation}
  \label{equation:delta_0}
  \bar\delta \defeq 1-\sqrt{\frac{1+2\gamma\bar\gamma-\bar\gamma^2-\sqrt{(1+2\gamma\bar\gamma-\bar\gamma^2)^2-4\gamma^2}}{2}}.
\end{equation}
Actually, we can make this bound cleaner (but slightly looser) by $(\bar\delta\le)1-\gamma<V<1$ (cf. Section~\ref{section:technicality}), which is convenient for interpretation.
In addition, we can check that $\kappa_0$ is strictly increasing on $V\in(\bar\delta,1)$ by calculating the derivative $\kappa_0'$.
We visualize $\kappa_0$ in Figure~\ref{fig:kappa_plot}.
Now we are ready to apply Lemma~\ref{lemma:alignment}.
Fix any $\Delta>\bar\delta$.
Since $\kappa_0$ is strictly increasing on $(\bar\delta,1)$, for any $\delta\ge\Delta(>\bar\delta)$ and $V\in[\delta,1]$,
we have $\langle{Z\alphabf},{P_{\ubf}^\perp\ubf_*}\rangle \ge \kappa_0(V) \ge \kappa_0(\delta) \ge \kappa_0(\Delta)>0$.
Setting $\kappa=\kappa_0(\Delta)$, which is a positive constant independent of $t$ and $\delta$,
Lemma~\ref{lemma:alignment} gives $T_1=O(\exp(\exp(-\delta/\kappa_0(\Delta)\gamma^2)))$ for all $\delta>\Delta$.
The total time in the early stage is $T_\text{escape} + T_1 = O(\exp(\exp(-\delta)))$,
concluding the proof.

\begin{remark}
\label{remark:kappa}
To further interpret $\kappa_0(\Delta)$, we bound it from below with, e.g., $\Delta=1-\gamma$, as follows:
\begin{equation}
  \kappa_0(1-\gamma)
  = \gamma - \gamma^2\bar\gamma - \gamma\sqrt{(1-\gamma^2)(1-\bar\gamma^2)}
  = \frac{\gamma(\gamma-\bar\gamma)^2}{1-\gamma\bar\gamma+\sqrt{(1-\gamma^2)(1-\bar\gamma^2)}}
  \ge \frac{\gamma(\gamma-\bar\gamma)^2}{2},
  \label{equation:kappa_lower}
\end{equation}
where the AM-GM inequality $2\sqrt{(1-\gamma^2)(1-\bar\gamma^2)}\le 2-\gamma^2-\bar\gamma^2$ is used.
This bound is positive as long as $\gamma\ne\bar\gamma=\tfrac1n\sum_i\inpr{\ubf_*}{\zbf_i}$, which holds broadly unless all $\zbf_i$ are the support vectors.
\end{remark}

\subsection{Lower bound}
\label{section:lower_bound}
So far, we confirmed the fast weak alignment within $T_1=O(\exp(\exp(-\delta)))$ time in Theorem~\ref{theorem:early_stage}.
We now establish a lower bound on the tangential alignment $V(t)$ during the early stage,
showing the tightness of our analysis.
Its proof largely shares the same idea as Lemma~\ref{lemma:alignment} (early-stage alignment),
to solve the differential inequality of $V(t)$,
but we need a better bound of $r(t)$ to get the matching lower bound.
The refined radial estimate in Lemma~\ref{lemma:l_r_bounds} serves for this purpose.
After all, we can establish a matching bound $V(t)\sim1-\log\log t$ during the early stage $t\le\Theta(\exp(\exp(-\delta)))$,
ensuring the tightness of the analysis.
The complete proof is deferred to Appendix~\ref{section:lower_bound_proof}.
\begin{tcolorbox}[highlightbox]
\begin{theorem}
  \label{lemma:lyapunov_lower_bound}
  Fix $\delta > 0$ and assume $L(0)<e^{\rho\gamma^2}$.
  Under the setup in Section~\ref{section:setup}, for $t\ge0$:
  \begin{equation}
    \label{equation:lyapunov_lower_bound}
    V(t)
    \ge V(0) - C_0 - 2\gamma^{-2}\sqrt{V(0)}\log\log(L(0)^{-1}+\gamma^2t),
  \end{equation}
  where $C_0$ is a problem-dependent constant.
  In addition, for $\delta>0$, $V(t)\ge \delta$ holds during $t\in[0,T_1]$ with $T_1=\Omega(\exp(\exp(-\delta)))$.
\end{theorem}
\end{tcolorbox}

\subsection{Case of logistic loss}
\label{section:logistic}
Theorem~\ref{theorem:early_stage} can be shown even if we use the logistic loss $\ell(m)\defeq\log(1+e^{-m})$.
While the main proof structure remains the same, (i) the self-bounding property $\gamma L\le \|\nabla L\|\le L$ and (ii) the temperature monotonicity $\inpr{\ubf}{Z\alphabf}$ used in Lemma~\ref{lemma:uniform_bound} are specific to the exponential loss so far.
Here, we explain how these two can be extended for the logistic loss.
With the margin $m_i(t) \defeq \inpr{\wbf(t)}{\zbf_i}$, we use the following notations:
\[
q_i(t) \defeq -\ell'(m_i(t)), \quad G(t) \defeq \frac1n\sum_{i=1}^nq_i(t), \text{~~and~~} \alpha_i(t) \defeq \frac{q_i(t)}{\sum_{j=1}^nq_j(t)}.
\]
The tangential Lyapunov function obeys $\dot V = -\inpr{Z\alphabf}{P_{\ubf}^\perp\ubf_*}\cdot G/r$ instead of \eqref{equation:dot_v}, with this new gradient weight $\alphabf\in\Delta^{n-1}$.

\begin{restatable}{lemma}{selfbdd}
  \label{lemma:logistic_self_bounding}
  For the logistic loss, define the problem-dependent constant
  \[
    \bar\rho_0 \defeq \frac{1-e^{-nL(0)}}{nL(0)} \text{~~and~~} \gamma_\ell \defeq \gamma\bar\rho_0.
  \]
  Then for all $t\ge 0$, we have $\bar\rho_0 L(t) \le G(t) \le L(t)$.
  Here $\gamma_\ell>0$ as long as $L(0)>0$.
  As a corollary, we have $\gamma_\ell L(t) \le \|\nabla L(t)\|\le L(t)$.
\end{restatable}
This modified self-bounding inequality leads to the risk estimates $\frac1{L(0)^{-1}+t}\le L(t) \le \frac{1}{L(0)^{-1}+\gamma_\ell^2t}$
and the radial estimates $|r(t) - \rho| \le \frac1{\gamma_\ell^2}\log(1+L(0)\gamma_\ell^2t)$,
replacing Lemma~\ref{lemma:l_r_bounds}.
This fix directly leads to the early weak-alignment rate $T_1=O(\exp(\exp(-\delta/\kappa\bar\rho_0\gamma_\ell^2)))$, under the ansatz~\eqref{equation:lyapunov_uniform_bound}, replacing Lemma~\ref{lemma:alignment}.
The ansatz needs the following modified temperature monotonicity to justify.
\begin{restatable}{lemma}{tempmono}
  \label{lemma:logistic_temperature_monotonicity}
  For the logistic loss, we have $\inpr{\ubf(t)}{Z\alphabf(t)} \le [\inpr{\ubf(t)}{\bar\zbf}]_+$ for all $t\ge 0$.
\end{restatable}
With the above fixes, we can show the same weak alignment rate for the logistic loss.
Its proof together with Lemmas~\ref{lemma:logistic_self_bounding}~and~\ref{lemma:logistic_temperature_monotonicity} are shown in Section~\ref{section:lemma_logistic}.
\begin{tcolorbox}[highlightbox]
\begin{corollary}
  \label{corollary:logistic}
  Under the same setup as Theorem~\ref{theorem:early_stage} except that the logistic loss is used, there exists a problem-dependent threshold $\bar\delta>0$ such that the following holds:
  pick a fixed constant $\Delta>\bar\delta$, then for any $\delta>\Delta$, the tangential alignment achieves $V(t)\le\delta$ after $t=O(\exp(\exp(-\delta/\kappa_0(\Delta)\bar\rho_0\gamma_\ell^2)))$ time.
\end{corollary}
\end{tcolorbox}

\section{Discrete-time analysis}
\label{section:discrete}
We now extend our early-stage flow analysis to the discrete-time gradient descent setting with stepsize $\eta>0$:
$\wbf(t+1)=\wbf(t)-\eta\nabla L(\wbf(t))$ for every $t\ge0$.
Similarly, decompose the parameter into the radial and tangential components, $r(t)\defeq\|\wbf(t)\|$ and $\ubf(t)\defeq\wbf(t)/\|\wbf(t)\|$, respectively,
and focus on the tangential alignment $V(t)\defeq1-\inpr{\ubf(t)}{\ubf_*}$.
The parameter is initialized with $L(0)$, $V(0)$, and $\rho\defeq r(0)$.
Note that these notations remain the same as the continuous case but should not be confused.
Here we show the overview of the discrete-time analysis and defer the full analysis to Appendix~\ref{section:discrete_full}.
In the discrete-time case, we need to carefully track the discretization error of the continuous flows,
and the stepsize choice plays an important role therein.
\begin{assumption}
  \label{assump:stepsize}
  The stepsize satisfies $\eta \le c_0\min\{\rho L(0)^{-1},\gamma^2\}$ for a sufficiently small $c_0>0$.
\end{assumption}
With this choice, we can show the discrete analogs of the structural results in Section~\ref{section:structure},
namely, the radial/tangential flows (Lemma~\ref{lemma:decomposition}) and risk/radial bounds (Lemma~\ref{lemma:l_r_bounds}).
We sketch both ideas.

\paragraph{Discrete radial/tangential flows.}
For $t=0,1,2,\dots$, we keep tracking the one-step increments $\Delta r(t)\defeq r(t+1)-r(t)$ and $\Delta\ubf(t)\defeq\ubf(t+1)-\ubf(t)$,
for which we can derive the following discretized dynamics (and the formal statement is presented in Lemma~\ref{lemma:discrete_increments}):
\[
  \begin{cases}
    \Delta r(t) &= -\eta\inpr{\ubf(t)}{\nabla L(\wbf(t))} - O\left(\frac{\eta^2L(\wbf(t))^2}{r(t)}\right), \\
    \Delta\ubf(t) &= -\frac{\eta}{r(t)}P_{\ubf(t)}^\perp\nabla L(\wbf(t)) + O\left(\frac{\eta^2L(\wbf(t)^2)}{r(t)^2}\right).
  \end{cases}
\]
Here, $\Delta r(t)/\eta$ and $\Delta\ubf(t)/\eta$ play similar roles to $\dot r(t)$ and $\dot\ubf(t)$ in the continuous flows~\eqref{equation:decomposition},
and the respective dynamics expressions remain similar up to the remainder terms.
To make these residuals negligible, we need the local condition $\eta L(\wbf(t))/r(t)=O(1)$.
Actually, Assumption~\ref{assump:stepsize} satisfies this at $t=0$.
For $t\ge1$, we can expect that $L(\wbf(t))=\Theta(t^{-1})$ holds and $r(t)$ has a non-trivial lower bound, which is similar to the continuous case shown in Lemma~\ref{lemma:l_r_bounds}---%
still, their discrete-time extensions need to be addressed carefully.
Then, it is natural to conjecture that $(\eta L(\wbf(t))/r(t))_{t\in\Zbb_{\ge0}}$ is a decreasing sequence, and thus the local condition persists throughout $t=1,2,\dots$.
This can be indeed verified via induction; see Remark~\ref{remark:stepsize_consequences} in Appendix~\ref{section:discrete_full}.

\paragraph{Discrete risk/radial bounds.}
As seen in Lemma~\ref{lemma:l_r_bounds}, we have similar bounds for the discrete case:
$L(t)=\Theta(t^{-1})$ and $r(t)\sim\rho+\log(1+t)$---%
these bounds in discrete time incur multiplicative discretization error up to the factor $1-O(\eta\gamma^{-2})$, which is negligible with the stepsize choice $\eta\le c_0\gamma^2$.
This error control suffices when taking the telescoping sum of $\Delta r(t)$ and $\Delta\ubf(t)$ (which replaces the differential inequalities we heavily relied on in the continuous-time analysis).
The full statement of the risk/radial bounds is in Lemma~\ref{lemma:discrete_l_r_bounds} in Appendix~\ref{section:discrete_full}.

\paragraph{Early-stage alignment in discrete time.}
Building upon the above, the remaining piece is not substantially complicated; the geometric lemma (Lemma~\ref{lemma:uniform_bound}) persists in the discrete case.
As in Section~\ref{section:main}, the discrete-time dynamics also undergoes the escape stage to get out from poor initialization, and the weak alignment stage to achieve up to $\inpr{\ubf}{\ubf_*}\ge1-\delta$.
Instead of solving the differential inequalities, we take the telescoping sums over the discrete dynamics.
Now we summarize the results informally.
\begin{itemize}
  \item Escape stage: after $O(1)$ iterations, the parameter achieves $V(t)\le1$.
  \item Weak alignment stage: after $O(\eta^{-1}\!\exp(\exp(-\delta)))$ more iterations, the parameter achieves $V(t)\le\delta$.
  The error tolerance $\delta$ is roughly up to $1-\sqrt\gamma$, as in the continuous case.
  \item Lower bound: to achieve $V(t)\le\delta$, the necessary iteration is no less than $\Omega(\eta^{-1}\!\exp(\exp(-\delta)))$.
\end{itemize}

\section{Discussion and open question}
\label{section:discussion}
\paragraph{Improvement upon error tolerance $\delta$.}
Despite our tight upper and lower bounds for the tangential alignment, the established weak alignment remains to be up to $\inpr{\ubf}{\ubf_*}\ge\gamma$ in general, or $\sqrt\gamma$ at best for some nice dataset.
Revisiting Figure~\ref{fig:simulation}, the early-stage alignment would be higher than the dataset margin $\gamma$ within considerably short time.
Improving this alignment error tolerance is an important open question.
To overcome this, we need to eschew invoking the uniform bound of the form $\langle Z\alphabf,P_{\ubf}^\perp\ubf_*\rangle\ge\kappa$ in \eqref{equation:lyapunov_uniform_bound}.
Our current proof in Lemma~\ref{lemma:uniform_bound} purely relies on a geometric argument of a dataset, but its improvement is apparently impossible.
Thus, we presumably need to rely on other arguments, such as an additional assumption on a dataset structure, or a probabilistic argument.

\paragraph{Analogy to ``edge-of-stability'' analysis.}
The edge of stability (EoS) is a phenomenon where learning dynamics with an excessively large stepsize oscillates significantly at the early stage yet eventually stabilizes to a flat minimum in a loss landscape, observed across deep learning and simpler problems~\citep{Cohen2021ICLR,Ahn2022ICML}.
Indeed, this also happens in logistic regression, and the dynamics undergoes the first oscillatory stage before stable convergence~\citep{Wu2024COLT,Bao2025NeurIPS}.
Their analysis framework is close to Section~\ref{section:setup}, assuming the linear separable case.
Although a large stepsize benefits the convergence rate of the optimization problem, training points are correctly classified once the early oscillatory stage terminates.
At this point, the parameter already correlates with the max-margin direction positively to some extent.
This observation inspires our study.
In contrast to these ``EoS'' analyses, our analysis deals with a small stepsize (Assumption~\ref{assump:stepsize}), and our tightness result hinges on this stepsize choice,
so it remains unclear what our analysis implies with a larger stepsize.
While intentionally controlled, discretization error could accelerate the parameter to decrease the tangential alignment error additionally---%
this can be indirectly supported via the one-step update of $V(t)$ in \eqref{equation:discrete_lyapunov_diff}.
To verify this, we need to go beyond this infinitesimal error control.

\section*{Acknowledgment}
HB is supported by JST PRESTO JPMJPR24K6.

\bibliographystyle{abbrvnat}
\bibliography{reference}

\appendix

\section{Technicalities in Section~\ref{section:proof_main}}
\label{section:technicality}

\paragraph{\underline{Tangential error $\bar\delta\le1-\gamma$.}}
Recall the quantity $\bar\delta$ defined in~\eqref{equation:delta_0}:
\[
  \bar\delta \defeq 1-\sqrt{\frac{1+2\gamma\bar\gamma-\bar\gamma^2-\sqrt{(1+2\gamma\bar\gamma-\bar\gamma^2)^2-4\gamma^2}}{2}},
\]
which is the lowest possible error tolerance of the tangential alignment in the early stage.
This satisfies $\bar\delta \le 1-\gamma$.
To show $\bar\delta \le 1-\gamma$, it suffices to show that
\[
  \sqrt{\frac{B-\sqrt{B^2-4\gamma^2}}{2}} \ge \gamma, \text{~~where~~}
  B\defeq 1+2\gamma\bar\gamma-\bar\gamma^2
\]
By squaring both sides and rearranging, this is equivalent to
\[
  B - 2\gamma^2 \ge \sqrt{B^2-4\gamma^2}.
\]
Squaring again (noting that $B - 2\gamma^2 \ge 0$; can be confirmed via $\gamma\le\bar\gamma\le1$), we obtain
\[
  (B - 2\gamma^2)^2 \ge B^2-4\gamma^2.
\]
Rearranging, we have $\gamma^2-B+1\ge0$,
for which we substitute $B = 1+2\gamma\bar\gamma-\bar\gamma^2$ and have
\[
  \gamma^2 - (1+2\gamma\bar\gamma-\bar\gamma^2) + 1 = (\gamma-\bar\gamma)^2 \ge 0.
\]
This inequality is always true, and thus we confirmed $\bar\delta\le1-\gamma$.

\paragraph{\underline{$\bar\delta$ is monotonically decreasing in $\bar\gamma$}.}
Using the same notation $B = 1 + 2\gamma\bar\gamma - \bar\gamma^2$, we can write
\[
  \bar\delta = 1 - \sqrt{\frac{B - \sqrt{B^2 - 4\gamma^2}}{2}}.
\]
To show $\bar\delta$ is decreasing in $\bar\gamma$, it suffices to show that $g(\bar\gamma) \defeq \frac{B - \sqrt{B^2 - 4\gamma^2}}{2}$ is increasing in $\bar\gamma$.

First, note that $\frac{\partial B}{\partial \bar\gamma} = 2\gamma - 2\bar\gamma = 2(\gamma - \bar\gamma) \le 0$ since $\gamma \le \bar\gamma$.
Also, $B > 0$ for $\gamma \le \bar\gamma \le 1$ can be shown by completing the square in $\bar\gamma$.
Now we compute
\[
  \frac{\partial g}{\partial \bar\gamma}
  = \frac{1}{2}\left(\frac{\partial B}{\partial \bar\gamma} - \frac{B}{\sqrt{B^2 - 4\gamma^2}}\frac{\partial B}{\partial \bar\gamma}\right)
  = \frac{1}{2}\frac{\partial B}{\partial \bar\gamma}\left(1 - \frac{B}{\sqrt{B^2 - 4\gamma^2}}\right).
\]
Since $B > 0$ and $B^2 - 4\gamma^2 < B^2$, we have $\sqrt{B^2 - 4\gamma^2} < B$, which implies $\frac{B}{\sqrt{B^2 - 4\gamma^2}} > 1$.
Thus $g$ is increasing in $\bar\gamma$, which implies $\bar\delta = 1 - \sqrt{g}$ is decreasing in $\bar\gamma$.

This fact indicates that the best tangential error that we can achieve in the early stage is $\bar\delta=1-\sqrt{g(1)}=1-\sqrt\gamma$ if a given dataset has a large average dataset margin $\bar\gamma\approx1$.
In this case, we can achieve up to $\inpr{\ubf}{\ubf_*}\ge\sqrt\gamma$.

\section{Auxiliary lemmas}
\label{section:auxiliary}

\begin{lemma}
  \label{lemma:integral_polylog}
  For $a,b,c>0$, assume $ac \le 1$. Then, we have
  \[
    \int_0^t\frac{\rd s}{(a+s)(b+\log(1+cs))}
    \ge \frac{1}{(1+ac)(b+\log2)} + \frac12\log\left(1+\frac{\log ct}{b+\log 2}\right).
  \]
\end{lemma}
\begin{proof}
  We split the interval of integration into $s\in[0,1/c]$ and $s\in(1/c,t]$.
  For the first interval, we have
  \[
    \int_0^{1/c}\frac{\rd s}{(a+s)(b+\log(1+cs))}
    \ge \int_0^{1/c}\frac{\rd s}{(a+\frac1c)(b+\log2)}
    = \frac{1}{(1+ac)(b+\log2)}.
  \]
  For the second interval, by noting $a\le 1/c\le s$, we have
  \[
    \begin{aligned}
      \int_{1/c}^t\frac{\rd s}{(a+s)(b+\log(1+cs))}
      &\ge \int_{1/c}^t\frac{\rd s}{2s(b+\log(2cs))} \\
      &= \frac12 \int_{1/c}^t\frac{\rd s}{s(\bar b+\log s)} && \text{($\bar b\defeq b+\log 2c$)} \\
      &= \frac12 \int_{\bar b-\log c}^{\bar b+\log t}\frac{\rd\varsigma}{\varsigma} \\
      &= \frac12\log\left(\frac{\bar b + \log t}{\bar b - \log c}\right),
    \end{aligned}
  \]
  where we use the change of variable $\varsigma\defeq\bar b+\log s$.
  Adding the two terms concludes the desired lower bound.
\end{proof}

\begin{lemma}
  \label{lemma:temperature_monotonicity}
  Fix $\mbf\in\Rbb^n$ and define $\alphabf\in\Delta^{n-1}$ by $\alpha_i\defeq e^{-\beta m_i}/\sum_{j=1}^ne^{-\beta m_j}$.
  Let $f(\beta)\defeq \sum_{i=1}^n\alpha_im_i$ for $\beta\ge0$.
  Then, $f$ is monotonically nonincreasing.
\end{lemma}
\begin{proof}
  We compute $f'(\beta)=\sum_{i=1}^nm_i\cdot\alpha_i'(\beta)$,
  where $\alpha_i'(\beta)=\rd{\alpha_i}/\rd{\beta}$.
  The derivative of $\alpha_i$ is
  \begin{align*}
    \frac{\rd{\alpha_i}}{\rd{\beta}}
    &= \frac{-m_ie^{-\beta m_i}\sum_{j=1}^ne^{-\beta m_j} - e^{-\beta m_i}\sum_{j=1}^n(-m_j)e^{-\beta m_j}}{(\sum_{j=1}^ne^{-\beta m_j})^2} \\
    &= \frac{e^{-\beta m_i}}{\sum_{j=1}^ne^{-\beta m_j}}\cdot\frac{-m_i\sum_{j=1}^ne^{-\beta m_j} + \sum_{j=1}^nm_je^{-\beta m_j}}{\sum_{j=1}^ne^{-\beta m_j}} \\
    &= \alpha_i\Bigl(-m_i + \sum_{j=1}^n\alpha_jm_j\Bigr) \\
    &= \alpha_i(f(\beta) - m_i).
  \end{align*}
  Therefore,
  \begin{align*}
    f'(\beta)
    = \sum_{i=1}^nm_i\alpha_i(f(\beta) - m_i)
    = -\sum_{i=1}^n\alpha_im_i^2 + \Bigl(\sum_{i=1}^n\alpha_im_i\Bigr)^2
    = -\mathrm{Var}_{\alphabf}(m_i) \le 0,
  \end{align*}
  where $\mathrm{Var}_{\alphabf}(m_i)\defeq\sum_{i=1}^n\alpha_im_i^2 - (\sum_{i=1}^n\alpha_im_i)^2$ denotes the variance of $m_i$ under the distribution $\alphabf\in\Delta^{n-1}$.
  Since $f'(\beta)\le 0$ for all $\beta\ge0$, the function $f$ is monotonically nonincreasing.
\end{proof}

\section{Continuous-time lower bound proof}
\label{section:lower_bound_proof}
Here we show the complete proof of Theorem~\ref{lemma:lyapunov_lower_bound}.
From \eqref{equation:dot_v}, we can derive the following differential inequality:
\[
  \dot V(t)
  = -\frac{L(\wbf(t))}{r(t)}\inpr{Z\alphabf}{P_{\ubf(t)}^\perp\ubf_*}
  \ge -\frac{L(\wbf(t))}{r(t)}\|Z\alphabf\|\|P_{\ubf(t)}^\perp\ubf_*\|
  \ge -\frac{L(\wbf(t))}{r(t)}\sqrt{2V(t)},
\]
where we used
\(
  \|P_{\ubf(t)}^\perp\ubf_*\|^2
  = \|\ubf_*\|^2 - \inpr{\ubf(t)}{\ubf_*}^2
  = 1 - (1-V(t))^2 = 2V(t) - V(t)^2
  \le 2V(t)
\).
Separating variables, for $V>0$ we have
\[
  \frac{\rd V}{\sqrt{V}} \ge -\sqrt{2}\frac{L(t)}{r(t)}\rd t
  \quad \overset{\text{integrate from $0$ to $t$}}\leadsto \quad
  \sqrt{V(t)} \ge \sqrt{V(0)} - \frac{1}{\sqrt{2}}\int_0^t\frac{L(s)}{r(s)}\rd s.
\]
We now bound the integral using the lower bounds on $r(t)$ from Lemma~\ref{lemma:l_r_bounds}.
By defining
\(
  T_0 \defeq {L(0)^{-1}\gamma^{-2}}\bigr[\exp\bigr(\frac{\rho+\log L(0)}{1+\gamma^{-2}}\bigl)-1\bigl]
\), we have
\begin{equation}
  \label{equation:r_lower_bound}
  r(t)\ge\begin{cases}
    \rho-\gamma^{-2}\log(1+L(0)\gamma^2t) & \text{if $t\in[0,T_0]$,} \\
    \log(L(0)^{-1}+\gamma^2t) & \text{if $t\ge T_0$,}
  \end{cases}
\end{equation}
At $t=T_0$, we have $r(T_0)\ge\frac{\rho\gamma^2-\log L(0)}{1+\gamma^2} \eqdef \underline{r} > 0$,
where the strict positivity holds from the assumption $L(0)<e^{\rho\gamma^2}$.
We split the integral into two intervals $s\in[0,T_0]\cup(T_0,t)$.
For the first interval $s\in[0,T_0]$, we combine $r(s)\ge\underline{r}$
with the upper bound of $L(s)$ from Lemma~\ref{lemma:l_r_bounds} to have
\[
  \int_0^{T_0}\frac{L(s)}{r(s)}\rd s
  \le \frac{1}{\underline{r}}\int_{0}^{T_0}\frac{\rd s}{L(0)^{-1}+\gamma^2s}
  = \frac{1}{\underline{r}\gamma^2} \log\left(\frac{L(0)^{-1}+\gamma^2T_0}{L(0)^{-1}}\right)
  = \frac{\rho+\log L(0)}{\rho\gamma^2-\log L(0)}.
\]
For the second interval $s\in(T_0,t]$, we combine \eqref{equation:r_lower_bound}
with the upper bound of $L(s)$ to have
\[
    \int_{T_0}^t\frac{L(s)}{r(s)}\rd s
    \le \int_{T_0}^t\frac{\rd s}{h(s)\log h(s)}
    = \frac{1}{\gamma^2}\int_{h(T_0)}^{h(t)}\frac{\rd\varsigma}{\varsigma}
    = \frac{1}{\gamma^2}\log\biggl(\frac{\log(L(0)^{-1}+\gamma^2t)}{\frac{\rho\gamma^2-\log L(0)}{1+\gamma^{-2}}}\biggr),
\]
where we use the change of variable $\varsigma\defeq h(s)\defeq\log(L(0)^{-1}+\gamma^2s)$.
Plugging them back yields
\[
  \begin{aligned}
    \sqrt{V(t)}
    &\ge \sqrt{V(0)} - \frac{1}{\sqrt 2}\left[
      \bar C_0
      + \frac{1}{\gamma^2}\log(\log(L(0)^{-1}+\gamma^2t))
    \right] \\
    \text{where} \quad &
    \bar C_0\defeq
    \frac{\rho+\log L(0)}{\rho\gamma^2-\log L(0)}
    - \frac{1}{\gamma^2}\log\left(\frac{\rho\gamma^2-\log L(0)}{1+\gamma^{-2}}\right)
  \end{aligned}
  .
\]
Squaring and relaxing the bound by $(a-b)^2\ge a^2-2ab$, we have the lower bound~\eqref{equation:lyapunov_lower_bound},
where $C_0\defeq \sqrt{2V(0)}\bar C_0$.
Lastly, to ensure $V(t)\ge\delta$, we need
$\text{(lower bound)}\ge\delta$,
which yields
\[
  t < 
    \frac{1}{\gamma^2}\exp\biggl[\exp\biggl(\frac{\gamma^2(V(0)-C_0-\delta)}{\sqrt{2V(0)}}\biggr)\biggr] - \frac{1}{L(0)\gamma^2}.
\]

\section{Extension to logistic loss}
\label{section:lemma_logistic}
In this section, we show the extension of the weak alignment analysis to the logistic loss.
Specifically, we show Lemmas~\ref{lemma:logistic_self_bounding}~and~\ref{lemma:logistic_temperature_monotonicity}.
Recall notations: with the margin $m_i(t) \defeq \inpr{\wbf(t)}{\zbf_i}$, we write
\[
q_i(t) \defeq -\ell'(m_i(t)), \quad G(t) \defeq \frac1n\sum_{i=1}^nq_i(t), \text{~~and~~} \alpha_i(t) \defeq \frac{q_i(t)}{\sum_{j=1}^nq_j(t)}.
\]

\selfbdd*
\begin{proof}
  We first derive the upper bound of $\|\nabla L(t)\|$:
  \[
    \|\nabla L(t)\|
    = \left\|\frac1n\sum_{i=1}^n(-\ell'(m_i(t))) \cdot \zbf_i\right\|
    = G(t) \cdot \left\|\sum_{i=1}^n\frac{q_i(t)}{nG(t)}\zbf_i\right\|
    = G(t) \cdot \|Z\alphabf(t)\|
    \le G(t).
  \]
  Furthermore, we have
  \[
    q_i(t) = \frac{e^{-m_i(t)}}{1+e^{-m_i(t)}} = \frac{e^{\ell(m_i(t))}-1}{e^{\ell(m_i(t))}}
    = 1 - e^{-\ell(m_i(t))}
    \le \ell(m_i(t)),
  \]
  which implies $G(t) \le L(t)$.

  To derive the lower bound of $G(t)$, we use the Weierstrass product inequality $\prod_ia_i \ge 1-\sum_i(1-a_i)$ for $a_i\in[0,1]$:
  with $a_i=e^{-\ell(m_i(t))}$, we have
  \[
    nG(t)
    = \sum_{i=1}^n(1-e^{-\ell(m_i(t))})
    \ge 1 - \prod_{i=1}^ne^{-\ell(m_i(t))}
    = 1 - e^{-nL(t)},
  \]
  which implies
  \[
    \frac{G(t)}{L(t)} \ge \frac{1-e^{-nL(t)}}{nL(t)}.
  \]
  Note that the map $x\mapsto\frac{1-e^{-x}}{x}$ is decreasing on $(0,\infty)$.
  As $L(t)$ is nonincreasing due to $\dot L=-\|\nabla L\|^2\le 0$, we have
  \[
    \frac{G(t)}{L(t)} \ge \frac{1-e^{-nL(0)}}{nL(0)} = \bar\rho_0.
  \]
  Finally, we have $G \ge \bar\rho_0 L$.
  The last corollary is immediate by noting that $\|\nabla L\| = G\|Z\alphabf\| \ge \gamma G \ge \gamma_\ell L$.
\end{proof}

\tempmono*
\begin{proof}
  Throughout the proof, we fix $t$ and drop its dependency, e.g., $m_i$ indicates $m_i(t) = \inpr{\wbf(t)}{\zbf_i}$.
  Fix a direction $\ubf\in\Sbb^{d-1}$ and write the normalized margin $M_i\defeq\inpr{\ubf}{\zbf_i}$ independent of $r$, so that $m_i=rM_i$ and
  \[
    \alpha_i = \frac{q_i}{\sum_jq_j} = \frac{\sigma(-rM_i)}{\sum_j\sigma(-rM_j)},
    \text{~~where~~}\sigma(m)=\frac{1}{1+e^{-m}}.
  \]
  Let $\varphi(m)\defeq \ell''(m)/(-\ell'(m)) = \sigma(m)$, which is nonnegative and nondecreasing.
  Define $\bar M(r) \defeq \inpr{\ubf}{Z\alphabf}$, which depends on $r$ through $\alphabf$.
  The goal of this lemma is to show $\bar M(r) \le [\bar M(0)]_+$ for $r\ge 0$.
  With an abuse of notation, we hereafter often write $\alphabf(r)$ and $q_i(r)$ to denote the dependency on $r$.

  By elementary algebra, we have
  \[
    \begin{aligned}
      \alpha_i'(r)
      &= \frac{q_i'(r)\sum_jq_j(r) - q_i(r)\sum_jq_j'(r)}{(\sum_jq_j(r))^2} \\
      &= \alpha_i(r)\frac{q_i'(r)}{q_i(r)} - \alpha_i(r)\frac{\sum_jq_j'(r)}{\sum_jq_j(r)} \\
      &= -\alpha_i(r)\varphi(m_i)M_i + \alpha_i(r)\sum_{j=1}^n\alpha_j(r)\varphi(m_j)M_j.
    \end{aligned}
  \]
  Then, by noting $\bar M=\sum_i\alpha_i M_i$, we have
  \[
    \begin{aligned}
      \bar M'(r)
      &= \sum_{i=1}^n\alpha_i'(r)M_i \\
      &= -\sum_{i=1}^n\alpha_i(r)\varphi(m_i)M_i^2 + \bar M(r) \sum_{i=1}^n\alpha_i(r)\varphi(m_i)M_i \\
      &= -\underbrace{\sum_{i=1}^n\alpha_i(r)\varphi(m_i)(M_i-\bar M(r))^2}_{\eqdef A(r)} - \bar M(r) \underbrace{\sum_{i=1}^n\alpha_i(r)\varphi(m_i)(M_i-\bar M(r))}_{\eqdef B(r)} \\
      &= -A(r) - \bar M(r)B(r).
    \end{aligned}
  \]
  Here, we have $A(r)\ge 0$ and $B(r)\ge 0$, where the latter satisfies $B(r)=\mathrm{Cov}_{\alphabf}(\varphi(m_i),M_i)\ge 0$
  because $\varphi(m_i)=\sigma(rM_i)$ is increasing in $M_i$, namely, $\varphi(m_i)$ and $M_i$ positively correlates.
  Then, we divide the cases as follows:
  \begin{itemize}
    \item If $\bar M(r) > 0$, then we have $\bar M'(r) \le 0$.
    \item If $\bar M(r) = 0$, then we have $\bar M'(r) \le 0$.
    \item (If $\bar M(r) < 0$, the sign of $\bar M'(r)$ is undetermined.)
  \end{itemize}
  Thus, as $r$ increases, $\bar M(r)$ never becomes positive for any $r \ge r_0$ once we have $\bar M(r_0)=0$, due to the continuity of $\bar M$.
  Together with the nonincreasing-ness of $\bar M(r)$ on $\{r:\bar M(r)>0\}$, the maximum of $\bar M$ is attained at $r=0$ when $\bar M(0) > 0$, and we have $\sup_{r\ge 0}\bar M(r) \le [\bar M(0)]_+$.
  Therefore, we arrive at the conclusion $\inpr{\ubf}{Z\alphabf} \le [\bar M(0)]_+ = [\inpr{\ubf}{\bar\zbf}]_+$, where $\bar\zbf = \frac1n\sum_i\zbf_i$.
\end{proof}

\subsection{Weak alignment result (Proof of Corollary~\ref{corollary:logistic})}
Finally, we derive the weak alignment time $t=O(\exp(\exp(-\delta/\kappa_0(\Delta)\bar\rho_0\gamma_\ell^2)))$.
The overall flow remains the same as Theorem~\ref{theorem:early_stage}, so we highlight a few important points only.
We begin with the tangential Lyapunov flow:
\[
  \dot V(t) = \frac{1}{r(t)}\inpr{\nabla L(\wbf(t))}{P_{\ubf(t)}^\perp\ubf_*}
  = -\frac{G(t)}{r(t)}\inpr{Z\alphabf(t)}{P_{\ubf(t)}^\perp\ubf_*}.
\]
By following Lemma~\ref{lemma:l_r_bounds}, we can get the bounds of $G(t)$ and $r(t)$ by using the modified self-bounding inequalities (Lemma~\ref{lemma:logistic_self_bounding}):
\[
  G(t) \ge \bar\rho_0 L(t) \ge \frac{\bar\rho_0}{L(0)^{-1}+t}, \qquad
  r(t) \le \rho+\frac{1}{\gamma_\ell^2}\log(1+L(0)\gamma_\ell^2t).
\]
By plugging the ansatz $\inpr{Z\alphabf}{P_{\ubf}^\perp\ubf_*}\ge \kappa$ and the above bounds into $V$-flow, we have
\[
  \dot V(t) \le -\frac{\kappa\bar\rho_0\gamma_\ell^2}{[L(0)^{-1}+t][\rho\gamma_\ell^2+\log(1+L(0)\gamma_\ell^2t)]}.
\]
Solving this, we get the weak alignment time $t=O(\exp(\exp(-\delta/\kappa\bar\rho_0\gamma_\ell^2)))$.

Finally, the ansatz $\inpr{Z\alphabf}{P_{\ubf}^\perp\ubf_*}\ge \kappa$ can be justified in the same way as Lemma~\ref{lemma:uniform_bound}:
by replacing the temperature monotonicity for the exponential loss (Lemma~\ref{lemma:temperature_monotonicity}) with the logistic loss (Lemma~\ref{lemma:logistic_temperature_monotonicity}),
Lemma~\ref{lemma:uniform_bound} is fixed minimally and we have
\[
  \begin{aligned}
    \inpr{Z\alphabf}{P_{\ubf}^\perp\ubf_*}
    & \ge \gamma - (1-V)[(1-V)\bar\gamma+\sqrt{(1-\bar\gamma^2)(2V-V^2)}]_+ \\
    & = \gamma - (1-V)[(1-V)\bar\gamma+\sqrt{(1-\bar\gamma^2)(2V-V^2)}]
  \end{aligned}
\]
where the positive part $[\cdot]_+$ can be dropped and the equality holds throughout for $V\le 1$.
Thus, the ansatz $\inpr{Z\alphabf}{P_{\ubf}^\perp\ubf_*}\ge \kappa$ is justified with $\kappa=\kappa_0(\Delta)$, giving the weak alignment time $t=O(\exp(\exp(-\delta/\kappa_0(\Delta)\bar\rho_0\gamma_\ell^2)))$.

\section{Discrete-time analysis (full version)}
\label{section:discrete_full}

We now extend our analysis to the discrete-time gradient descent setting with stepsize $\eta>0$.
Recall that the gradient descent iterates $\wbf(t+1)=\wbf(t)-\eta\nabla L(\wbf(t))$ for $t=0,1,2,\dots$
and we decompose the parameter into $r(t)\defeq\|\wbf(t)\|$ (radial) and $\ubf(t)\defeq\wbf(t)/r(t)$ (tangential),
where we use the same notation $\wbf(t)$, $r(t)$, and $\ubf(t)$ across the continuous-time and discrete-time cases but it should not be confused.
The key challenge in discrete-time analysis is to control the nonlinear terms arising from the normalization $\ubf(t)=\wbf(t)/\|\wbf(t)\|$ after each update.
We follow the direct analysis approach by explicitly bounding the discrete updates and showing that they match the continuous-time dynamics up to controlled error terms.
To this end, we impose the stepsize choice in Assumption~\ref{assump:stepsize}, $\eta\le c_0\min\{\rho L(0)^{-1},\gamma^2\}$ for a sufficiently small $c_0>0$.
See Remark~\ref{remark:stepsize_consequences} below about how this stepsize choice controls discretization error.

Before deriving the discrete analog of Lemma~\ref{lemma:decomposition}, the discrete-time dynamics are shown first.
\begin{lemma}[Discrete-time dynamics]
  \label{lemma:discrete_decomposition}
  For the gradient descent $\wbf(t+1)=\wbf(t)-\eta\nabla L(\wbf(t))$, the radial and tangential components evolve as follows:
  \begin{align}
    r(t+1)^2 &= r(t)^2 - 2\eta r(t)\inpr{\ubf(t)}{\nabla L(\wbf(t))} + \eta^2\|\nabla L(\wbf(t))\|^2, \label{equation:discrete_radial_squared} \\
    \ubf(t+1) &= \frac{r(t)\ubf(t) - \eta\nabla L(\wbf(t))}{r(t+1)}. \label{equation:discrete_tangential}
  \end{align}
\end{lemma}
\begin{proof}
  From $\wbf(t+1)=\wbf(t)-\eta\nabla L(\wbf(t))$, we have
  \[
    r(t+1)^2 = \|\wbf(t)-\eta\nabla L(\wbf(t))\|^2
    = \|\wbf(t)\|^2 - 2\eta\inpr{\wbf(t)}{\nabla L(\wbf(t))} + \eta^2\|\nabla L(\wbf(t))\|^2,
  \]
  which yields \eqref{equation:discrete_radial_squared} by substituting $\wbf(t)=r(t)\ubf(t)$.
  The tangential update \eqref{equation:discrete_tangential} immediately follows from $\ubf(t+1)=\wbf(t+1)/r(t+1)$.
\end{proof}
Instead of the time derivatives $\dot r(t)$ and $\dot\ubf(t)$ in the continuous case, we focus on the following discrete-time increments:
\[
  \begin{cases}
    \Delta r(t) &\defeq r(t+1)-r(t), \\
    \Delta\ubf(t) &\defeq \ubf(t+1)-\ubf(t).
  \end{cases}
\]
The following lemma provides useful bounds on these increments, which are parallel results to the continuous-time dynamics shown in Lemma~\ref{lemma:decomposition},
by controlling the discretization error.

\begin{lemma}[Radial/tangential increments]
  \label{lemma:discrete_increments}
  Assume the stepsize satisfies $\eta\|\nabla L(\wbf(t))\| \le \frac{r(t)}{5}$.
  Then, the following hold:
  \begin{align}
    \label{equation:delta_r}
    \Delta r(t) &= -\eta\inpr{\ubf(t)}{\nabla L(\wbf(t))} - O\left(\frac{\eta^2\|\nabla L(\wbf(t))\|^2}{r(t)}\right), \\
    \label{equation:delta_u}
    \Delta\ubf(t) &= -\frac{\eta}{r(t)}P_{\ubf(t)}^\perp\nabla L(\wbf(t)) + O\left(\frac{\eta^2\|\nabla L(\wbf(t))\|^2}{r(t)^2}\right).
  \end{align}
\end{lemma}
\begin{proof}
  To evaluate the radial increment, we use the Taylor expansion $\sqrt{1+x}=1+x/2-x^2/8+O(x^3)$ for $|x|\le1/2$.
  From \eqref{equation:discrete_radial_squared}, we have
  \[
    r(t+1) = r(t)\sqrt{1 - \frac{2\eta\inpr{\ubf(t)}{\nabla L(\wbf(t))}}{r(t)} + \frac{\eta^2\|\nabla L(\wbf(t))\|^2}{r(t)^2}}.
  \]
  Under the assumption $\eta\|\nabla L(\wbf(t))\| \le r(t)/5$, we have
  \[
    \left|\frac{-2\eta\inpr{\ubf(t)}{\nabla L(\wbf(t))}}{r(t)} + \frac{\eta^2\|\nabla L(\wbf(t))\|^2}{r(t)^2}\right|
    \le \frac{2\eta\|\nabla L(\wbf(t))\|}{r(t)} + \frac{\eta^2\|\nabla L(\wbf(t))\|^2}{r(t)^2} \le \frac{1}{2},
  \]
  which allows us to apply the Taylor expansion to obtain
  \[
    \begin{aligned}
      r(t+1)
      &= r(t)\left[1 - \frac{\eta\inpr{\ubf(t)}{\nabla L(\wbf(t))}}{r(t)} + \frac{\eta^2\|\nabla L(\wbf(t))\|^2}{2r(t)^2} - O\left(\frac{\eta^2\|\nabla L(\wbf(t))\|^2}{r(t)^2}\right)\right] \\
      &= r(t) - \eta\inpr{\ubf(t)}{\nabla L(\wbf(t))} - O\left(\frac{\eta^2\|\nabla L(\wbf(t))\|^2}{r(t)}\right),
    \end{aligned}
  \]
  which yields \eqref{equation:delta_r}.

  To evaluate the tangential increment,
  we need to evaluate $(r(t+1))^{-1}=[r(t)+\Delta r(t)]^{-1}$ by the Taylor expansion $\frac{1}{1+x}=1-x+O(x^2)$ as follows:
  \[
    \frac{1}{r(t)+\Delta r(t)} = \frac{1}{r(t)} \cdot \frac{1}{1+\frac{\Delta r(t)}{r(t)}} = r(t)^{-1}\left[1-\frac{\Delta r(t)}{r(t)}+O\left(\frac{\Delta r(t)^2}{r(t)^2}\right)\right].
  \]
  This expansion leads the tangential dynamics~\eqref{equation:discrete_tangential} to
  \[
    \begin{aligned}
      \ubf(t+1)
      &= \frac{r(t)}{r(t+1)}\ubf(t) - \frac{\eta}{r(t+1)}\nabla L(\wbf(t)) \\
      &= \left[1-\frac{\Delta r(t)}{r(t)}+O\left(\frac{\Delta r(t)^2}{r(t)^2}\right)\right]\ubf(t) - \frac{\eta}{r(t)}\left[1-\frac{\Delta r(t)}{r(t)}+O\left(\frac{\Delta r(t)^2}{r(t)^2}\right)\right]\nabla L(\wbf(t)) \\
      &= \ubf(t) - \frac{\eta}{r(t)}\nabla L(\wbf(t)) - \frac{\Delta r(t)}{r(t)}\ubf(t) + O\left(\frac{\eta^2\|\nabla L(\wbf(t))\|^2}{r(t)^2}\right).
    \end{aligned}
  \]
  Substituting the radial increment~\eqref{equation:delta_r} into the third term, we obtain
  \[
    \frac{\Delta r(t)}{r(t)}\ubf(t)
    = -\frac{\eta}{r(t)}\inpr{\ubf(t)}{\nabla L(\wbf(t))}\ubf(t) - O\left(\frac{\eta^2\|\nabla L(\wbf(t))\|^2}{r(t)^2}\right).
  \]
  Plugging this back, we have
  \[
    \begin{aligned}
      \ubf(t+1)
      &= \ubf(t) - \frac{\eta}{r(t)}\left[\nabla L(\wbf(t)) - \inpr{\ubf(t)}{\nabla L(\wbf(t))}\ubf(t)\right] + O\left(\frac{\eta^2\|\nabla L(\wbf(t))\|^2}{r(t)^2}\right) \\
      &= \ubf(t) - \frac{\eta}{r(t)}P_{\ubf(t)}^\perp\nabla L(\wbf(t)) + O\left(\frac{\eta^2\|\nabla L(\wbf(t))\|^2}{r(t)^2}\right),
    \end{aligned}
  \]
  which yields \eqref{equation:delta_u}.
\end{proof}

\subsection{Discrete-time risk and radial bounds}

We now establish discrete-time analogs of the risk and radial bounds in Lemma~\ref{lemma:l_r_bounds}.
These bounds control the decay of the risk and growth of the radial norm throughout the discrete gradient descent dynamics.

\begin{lemma}[Discrete-time loss and radial bounds]
  \label{lemma:discrete_l_r_bounds}
  Let $\rho\defeq r(0)$ and assume $L(0)<e^{\rho\gamma^2}$.
  For discrete-time gradient descent with $t\ge0$, we have
  \begin{gather}
    \label{equation:discrete_l_bounds}
    \frac{1}{L(0)^{-1}+\eta t} \lesssim L(\wbf(t)) \lesssim \frac{1}{L(0)^{-1}+\eta\gamma^2 t}, \\
    \label{equation:discrete_r_bounds_1}
    \rho-\frac{1}{\gamma^2}\log(1+L(0)\eta\gamma^2t) \lesssim r(t) \lesssim \rho+\frac{1}{\gamma^2}\log(1+L(0)\eta\gamma^2 t), \\
    \label{equation:discrete_r_bounds_2}
    r(t) \gtrsim \log(L(0)^{-1}+\eta\gamma^2t),
  \end{gather}
  where $O(1)$-multiplicative negligible error is hidden in ``$\lesssim$''.
\end{lemma}
\begin{proof}
  \textbf{Upper bound on $L(\wbf(t))$.}
  By expanding the risk $L(\wbf(t))$ along the gradient descent dynamics, we have
  \[
    \begin{aligned}
      L(\wbf(t+1))
      &= \frac{1}{n}\sum_{i=1}^n e^{-\inpr{\wbf(t)-\eta\nabla L(\wbf(t))}{\zbf_i}}
      = \frac{1}{n}\sum_{i=1}^n e^{-\inpr{\wbf(t)}{\zbf_i}} e^{\eta\inpr{\nabla L(\wbf(t))}{\zbf_i}} \\
      &= \frac1n\sum_{i=1}^n e^{-\inpr{\wbf(t)}{\zbf_i}}\left[1+\eta\inpr{\nabla L(\wbf(t))}{\zbf_i} + O(\eta^2\|\nabla L(\wbf(t))\|^2)\right] \\
      &= L(\wbf(t)) - \eta\|\nabla L(\wbf(t))\|^2 + O(\eta^2\|\nabla L(\wbf(t))\|^2),
    \end{aligned}
  \]
  where the Taylor expansion $e^x = 1 + x + O(x^2)$ is used.
  By the self-bounding property~\eqref{equation:self_bounding}, we have $\gamma L(\wbf(t)) \le \|\nabla L(\wbf(t))\|$, which yields
  \begin{equation}
    \label{equation:discrete_l_one_step}
    L(\wbf(t+1)) \le L(\wbf(t)) - \eta\gamma^2 L(\wbf(t))^2 + O(\eta^2L(\wbf(t))^2).
  \end{equation}
  Now, we invert both sides and take the telescoping sum:
  \[
    \begin{aligned}
      \frac{1}{L(\wbf(t+1))}
      &\ge \frac{1}{L(\wbf(t))[1-\eta\gamma^2L(\wbf(t))+O(\eta^2L(\wbf(t)))]} \\
      &\ge \frac{1}{L(\wbf(t))}\left[1+\eta\gamma^2L(\wbf(t))-O(\eta^2L(\wbf(t)))\right] \\
      &= \frac{1}{L(\wbf(t))} + \eta\gamma^2\cdot \underbrace{(1 - O(\eta\gamma^{-2}))}_\text{$O(1)$-discretization error} \\
      &\gtrsim \frac{1}{L(\wbf(t))} + \eta\gamma^2,
    \end{aligned}
  \]
  where we use $\frac{1}{1-x}\ge1+x$ (when $0<x<1/2$) at the second inequality
  and the stepsize choice $\eta\le c_0\gamma^2$ at the last inequality.
  Telescoping yields
  \[
    \frac{1}{L(\wbf(t))} - \frac{1}{L(0)} \gtrsim \eta\gamma^2 t
    \quad \leadsto \quad
    L(\wbf(t)) \lesssim \frac{1}{L(0)^{-1}+\eta\gamma^2 t},
  \]
  where we used $\frac{1}{1-x}=1+x-O(x^2)$ in the second argument.

  \paragraph{Lower bound on $L(\wbf(t))$.}
  Similarly, the risk expansion with $e^x\ge1+x$ leads to
  \[
    L(\wbf(t+1))
    \ge \frac1n\sum_{i=1}^ne^{-\inpr{\wbf(t)}{\zbf_i}}\left[1+\eta\inpr{\nabla L(\wbf(t))}{\zbf_i}\right]
    \ge L(\wbf(t)) - \eta L(\wbf(t))^2,
  \]
  where the self-bounding property $\|\nabla L(\wbf)\|\le L(\wbf(t))$~\eqref{equation:self_bounding} is used.
  Taking the reciprocals with $\frac{1}{1-x}=1+x+O(x^2)$ (when $0<x<1/2$), we have
  \[
    \begin{aligned}
      \frac{1}{L(\wbf(t+1))}
      &\le \frac{1}{L(\wbf(t))[1-\eta L(\wbf(t))]} \\
      &= \frac{1}{L(\wbf(t))}\left[1+\eta L(\wbf(t)) + O(\eta^2L(\wbf(t))^2)\right] \\
      &= \frac{1}{L(\wbf(t))} + \eta + O(\eta^2L(\wbf(t))).
    \end{aligned}
  \]
  Taking the telescoping sum from $0$ to $t-1$ yields
  \[
    \begin{aligned}
      \frac{1}{L(\wbf(t))}
      &\le \frac{1}{L(0)} + \eta t + O(\eta^2)\sum_{s=0}^{t-1}L(\wbf(s)) \\
      &\lesssim \frac{1}{L(0)} + \eta t + O(\eta^2)\cdot \int_0^t\frac{\rd s}{L(0)^{-1}+\eta\gamma^2s} \\
      &= \frac{1}{L(0)} + \eta t + O(\eta) \cdot \frac{1}{\gamma^2}\log(1+L(0)\eta\gamma^2t) \\
      &\lesssim \frac{1}{L(0)} + \eta t,
    \end{aligned}
  \]
  where we used the upper bound of $L(\wbf(s))$ already derived.
  Lastly, taking the reciprocal gives
  \[
    L(\wbf(t)) \gtrsim \frac{1}{L(0)^{-1}+\eta t} = \frac{1}{L(0)^{-1}+\eta t}.
  \]

  \paragraph{Upper bound on $r(t)$.}
  By plugging $\inpr{\ubf}{\nabla L(\wbf)}\ge-\|\nabla L(\wbf)\|\ge -L(\wbf)$ (the second inequality follows from the self-bounding property~\eqref{equation:self_bounding}) into Lemma~\ref{lemma:discrete_increments},
  we have
  \[
    \Delta r(s)
    \le \eta L(\wbf(s)) - O\left(\frac{\eta^2\|\nabla L(\wbf(s))\|^2}{r(s)}\right)
    \lesssim \eta L(\wbf(s)),
  \]
  where the second inequality uses the stepsize choice and self-bounding $\|\nabla L(\wbf)\|\le L(\wbf)$.
  Taking the telescoping sum from $0$ to $t-1$ leads to
  \[
    \begin{aligned}
      r(t) &\lesssim \rho + \eta\sum_{s=0}^{t-1}L(\wbf(s)) \\
      &\le \rho + \eta\sum_{s=0}^{t-1}\frac{1}{L(0)^{-1}+\eta\gamma^2s} \\
      &\le \rho + \eta\int_0^t\frac{\rd s}{L(0)^{-1}+\eta\gamma^2s} \\
      &= \rho + \frac{1}{\gamma^2}\log(1+L(0)\eta\gamma^2t),
    \end{aligned}
  \]
  where we used the upper bound of $L(\wbf(s))$ in \eqref{equation:discrete_l_bounds}.

  \paragraph{Lower bounds on $r(t)$.}
  Similarly, by plugging $\inpr{\ubf}{\nabla L(\wbf)}\le\|\nabla L(\wbf)\|\le L(\wbf)$ and the stepsize choice into Lemma~\ref{lemma:discrete_increments}, we have
  \[
    \Delta r(s)
    \ge -\eta L(\wbf(s)) -O\left(\frac{\eta^2\|\nabla L(\wbf(s))\|^2}{r(s)}\right)
    \gtrsim -\eta L(\wbf(s)).
  \]
  Taking the telescoping sum from $0$ to $t-1$ leads to
  \[
    r(t)
    \gtrsim \rho - \eta\sum_{s=0}^{t-1}L(\wbf(s))
    \ge \rho - \frac{1}{\gamma^2}\log(1+L(0)\eta\gamma^2t).
  \]
  The last bound~\eqref{equation:discrete_r_bounds_2} can be similarly shown to Lemma~\ref{lemma:l_r_bounds} and hence we omit the proof.
\end{proof}

\begin{remark}
  \label{remark:stepsize_consequences}
  At this point, we revisit the stepsize choice $\eta \le c_0\min\{\rho L(0)^{-1},\gamma^2\}$ in Assumption~\ref{assump:stepsize}.
  This choice ensures the following two conditions hold throughout the discrete-time analysis:
  \begin{enumerate}
    \item \textbf{Local gradient condition:} $\eta\|\nabla L(\wbf(t))\| \le r(t)/5$ for all $t\ge0$.
    This condition was used in Lemma~\ref{lemma:discrete_increments} to establish the radial/tangential increments formula.
    We verify this by induction.
    At $t=0$, the self-bounding property~\eqref{equation:self_bounding} gives $\eta\|\nabla L(\wbf(0))\| \le \eta L(0) \le c_0\rho < \rho/5$.
    For the inductive step, suppose $\eta L(\wbf(t)) \le c_0 r(t)$ holds at step $t$.
    By Lemma~\ref{lemma:discrete_increments}, we have
    \[
      r(t+1) \ge r(t) - \eta L(\wbf(t)) - O\left(\frac{\eta^2 L(\wbf(t))^2}{r(t)}\right)
      \ge r(t)\left(1 - c_0 - O(c_0^2)\right).
    \]
    By the one-step update in \eqref{equation:discrete_l_one_step} (see the proof of Lemma~\ref{lemma:discrete_l_r_bounds}),
    \[
      \begin{aligned}
        L(\wbf(t+1))
        &\le L(\wbf(t))\left(1 - \eta\gamma^2 L(\wbf(t)) + O(\eta^2 L(\wbf(t)))\right) \\
        &= L(\wbf(t))\left(1 - \eta\gamma^2 L(\wbf(t))(1 - O(\eta\gamma^{-2}))\right).
      \end{aligned}
    \]
    A sufficiently small $c_0$ (and hence small $\eta\gamma^{-2}$) yields $1 - O(\eta\gamma^{-2}) > 0$, so $L(\wbf(t+1)) \le L(\wbf(t))$.
    Thus, for a sufficiently small $c_0$, the induction can be closed as follows:
    \[
      \frac{\eta L(\wbf(t+1))}{r(t+1)} \le \frac{\eta L(\wbf(t))}{r(t)(1 - c_0 - O(c_0^2))}
      \le \frac{c_0}{1 - c_0 - O(c_0^2)} < \frac15.
    \]

    \item \textbf{Discretization error:}
    The risk/radial bounds, originally shown in Lemma~\ref{lemma:l_r_bounds} for the gradient flow, incur discretization error in Lemma~\ref{lemma:discrete_l_r_bounds} for the discrete-time dynamics, which is multiplicatively up to $1-O(\eta\gamma^{-2})$ and thus can be negligible with the stepsize choice $\eta\le c_0\gamma^2$.
    Subsequently, we take the telescoping sum in Sections~\ref{section:discrete_ub}~and~\ref{section:discrete_lb}, where this discretization error is not propagated.
  \end{enumerate}
\end{remark}

\subsection{Upper bound}
\label{section:discrete_ub}
To establish the early-stage alignment, we show the weak alignment and escape stages, respectively, as in the continuous case.
First, we show the discrete analog of Lemma~\ref{lemma:alignment} (weak alignment).
\begin{lemma}
  \label{lemma:discrete_alignment}
  For an absolute constant $\kappa>0$,
  assume the following uniform lower bound holds during $t=0,1,\dots,T_1$:
  \(
    \langle{Z\alphabf(t)},{P_{\ubf(t)}^\perp\ubf_*}\rangle \ge \kappa
  \).
  Then, we have
  \[
    V(t) \le
    V(0) - \frac{\kappa\gamma^2}{2}\log\left(1 + \frac{\log(L(0)\eta\gamma^2t)}{\rho\gamma^2+\log2}\right)
  \]
  for $t=0,1,\dots,T_1$.
  For $\delta>0$, if the uniform bound holds during $t\in[0,T_1]$ with $T_1=O(\eta^{-1}\exp(\exp(-\delta/\kappa\gamma^2)))$,
  we have $V(t)\le\delta$ after $T_1$ iterations.
\end{lemma}
\begin{proof}
  By Lemma~\ref{lemma:discrete_increments}, the discrete increment of the tangential alignment $V(t)$ can be written as
  \begin{equation}
    \label{equation:discrete_lyapunov_diff}
    \begin{aligned}
      \Delta V(t)
      &\defeq V(t+1)-V(t) = -\inpr{\Delta\ubf(t)}{\ubf_*} \\
      &= \frac{\eta}{r(t)}\inpr{\nabla L(\wbf(t))}{P_{\ubf(t)}^\perp\ubf_*} - O\left(\frac{\eta^2\|\nabla L(\wbf(t))\|^2}{r(t)^2}\right) \\
      &= -\frac{\eta L(\wbf(t))}{r(t)}\inpr{Z\alphabf(t)}{P_{\ubf(t)}^\perp\ubf_*} - O\left(\frac{\eta^2\|\nabla L(\wbf(t))\|^2}{r(t)^2}\right) \\
      &\lesssim -\frac{\kappa\eta L(\wbf(t))}{r(t)},
    \end{aligned}
  \end{equation}
  where the uniform bound $\langle{Z\alphabf},{P_{\ubf}^\perp\ubf_*}\rangle\ge\kappa$ and the stepsize choice are collectively used at the last inequality.
  With Lemma~\ref{lemma:discrete_l_r_bounds}, the telescoping sum is taken for $L(\wbf(s))/r(s)$ as follows:
  \[
    \begin{aligned}
      \sum_{s=0}^{t-1}\frac{L(\wbf(s))}{r(s)}
      &\gtrsim \sum_{s=0}^{t-1} \frac{1}{(L(0)^{-1}+\eta s)[\rho+\gamma^{-2}\log(1+L(0)\eta\gamma^2 s)]} \\
      &\ge \frac{\gamma^2}{\eta} \int_0^t\frac{\rd s}{(L(0)^{-1}\eta^{-1}+s)[\rho\gamma^2+\log(1+L(0)\eta\gamma^2 s)]} \\
      &\ge \frac{\gamma^2}{\eta}\frac{1}{2(\rho\gamma^2+\log(1+\gamma^2))} + \frac{\gamma^2}{2\eta}\log\left(1+\frac{\log(L(0)\eta\gamma^2t)}{\rho\gamma^2+\log2}\right)
      ,
    \end{aligned}
  \]
  where we use Lemma~\ref{lemma:integral_polylog} with $a\defeq L(0)^{-1}\eta^{-1}$, $b\defeq\rho\gamma^2$, and $c\defeq L(0)\eta\gamma^2$.
  Therefore, we have
  \[
    \begin{aligned}
      V(t)
      &\lesssim V(0) - \frac{\kappa\gamma^2}{2(\rho\gamma^2+\log(1+\gamma^2))}
        - \frac{\kappa\gamma^2}{2}\log\left(1 + \frac{\log(L(0)\eta\gamma^2t)}{\rho\gamma^2+\log2}\right) \\
      &\le V(0) - \frac{\kappa\gamma^2}{2}\log\left(1 + \frac{\log(L(0)\eta\gamma^2t)}{\rho\gamma^2+\log2}\right).
    \end{aligned}
  \]
  To achieve $V(t)\le\delta$, we need $\text{(upper bound)}\le\delta$, which yields
  \[
    t \gtrsim \frac{1}{L(0)\eta\gamma^2}\exp\left\{
      (\rho\gamma^2+\log2)\left[\exp\left(\frac{2(V(0)-\delta)}{\kappa\gamma^2}\right) - 1\right]
    \right\}
    = O\left(\frac{\exp(\exp(-\delta/\kappa\gamma^2))}{\eta}\right).
  \]
\end{proof}

Next, we show the discrete analog of Lemma~\ref{lemma:escape} (escape from poor initialization).
\begin{lemma}
  \label{lemma:discrete_escape}
  Let
  $W(t)\defeq\inpr{\wbf(t)}{\ubf_*}=r(t)(1-V(t))$ be the parameter projection, and
  \(
    T_\text{escape}=\frac{1}{\eta L(0)}\left\{\exp(\rho\gamma^{-1}[V(0)-1])-1\right\}
  \).
  If $V(0)>1$, then the parameter eventually escapes from the bad initialization, namely, $V(t)\le 1$ holds after $t\ge T_\text{escape}$ iterations.
\end{lemma}
\begin{proof}
  With $\inpr{\ubf_*}{Z\alphabf}\ge\gamma$ and Lemma~\ref{lemma:discrete_l_r_bounds},
  we derive the discrete increment of $W(t)$ as follows.
  From the gradient descent update $\wbf(t+1)=\wbf(t)-\eta\nabla L(\wbf(t))$, we have
  \[
    W(t+1) - W(t) = -\eta\inpr{\nabla L(\wbf(t))}{\ubf_*} = \eta L(\wbf(t))\inpr{\ubf_*}{Z\alphabf(t)}
    \gtrsim \frac{\eta\gamma}{L(0)^{-1}+\eta t}.
  \]
  Summing from $s=0$ to $s=t-1$, we have
  \[
    W(t) \gtrsim W(0) + \eta\gamma\sum_{s=0}^{t-1}\frac{1}{L(0)^{-1}+\eta s}
    \ge W(0) + \eta\gamma\int_0^{t}\frac{\rd s}{L(0)^{-1}+\eta s}
    \ge W(0) + \gamma\log(1+\eta L(0)t).
  \]
  To escape from poor initialization (or to have $V(t)\le1$),
  it is sufficient to have $W(t)=r(t)(1-V(t))\ge0$.
  Hence, by solving $W(0)+\gamma\log(1+\eta L(0)t)\ge0$,
  we have $t\ge T_\text{escape}$.
\end{proof}

We now state the upper bound of the tangential alignment for discrete-time gradient descent.
The parameter first escapes from a bad initialization (Lemma~\ref{lemma:discrete_escape}; escape stage), and then weakly aligns with the max-margin direction (Lemma~\ref{lemma:discrete_alignment}; weak alignment stage).
\begin{theorem}
  \label{theorem:discrete_early_stage}
  Under Assumption~\ref{assump:linearly_separable}, consider the discrete-time gradient descent $\wbf(t+1)=\wbf(t)-\eta\nabla L(\wbf(t))$.
  There exists a problem-dependent threshold $\bar\delta>0$ such that the following holds:
  pick a fixed constant $\Delta>\bar\delta$, then
  for any $\delta>\Delta$,
  the tangential alignment achieves $V(t)\le\delta$ after $t=O(\eta^{-1}\exp(\exp(-\delta/\kappa_0(\Delta)\gamma^2)))$ iterations.
\end{theorem}
\begin{proof}
  The proof is the same as the continuous case, Theorem~\ref{theorem:early_stage}.
  Note that our geometric lemma (Lemma~\ref{lemma:uniform_bound}) carries over to the discrete case.
  Therefore, Lemma~\ref{lemma:discrete_escape} establishes the escape stage, and the weak alignment stage follows by combining Lemmas~\ref{lemma:discrete_alignment}~and~\ref{lemma:uniform_bound}.
\end{proof}
Theorem~\ref{theorem:discrete_early_stage} confirms that the early-stage weak alignment phenomenon observed in continuous time persists in the discrete-time setting.
The two-stage structure---escaping followed by weak alignment---is identical to the continuous-time case.

\subsection{Lower bound}
\label{section:discrete_lb}
We now establish a lower bound on the tangential alignment $V(t)$ during the early stage for discrete-time gradient descent.
Again we establish a matching lower bound, ensuring the tightness of the analysis in the discrete-time setting.
In this lower bound, we show the least elapsed time necessary for achieving weak alignment $V(t)\approx 1-\gamma$.
This error tolerance is reasonably chosen for the early stage; see the discussion in Section~\ref{section:proof_main}.
\begin{theorem}
  \label{theorem:discrete_lyapunov_lower_bound}
  Fix $\delta > 0$ and assume $L(0)<e^{\rho\gamma^2}$.
  Under Assumption~\ref{assump:linearly_separable}, for $t=0,1,\dots,T_1$,
  the parameter following the discrete-time gradient descent $\wbf(t+1)=\wbf(t)-\eta\nabla L(\wbf(t))$ satisfies
  \begin{equation}
    \label{equation:discrete_lyapunov_lower_bound}
    V(t)
    \ge V(0) - C - 2\gamma^{-2}\sqrt{V(0)}\log(\log(L(0)^{-1}+\eta\gamma^2t)),
  \end{equation}
  where $C > 0$ is a problem-dependent constant.
  In addition, for $\delta>1-\gamma$, $V(t)\ge \delta$ holds during $t\in[0,T_1]$ with $T_1=\Omega(\eta^{-1}\exp(\exp(-\delta)))$.
\end{theorem}
\begin{proof}
  From the discrete increment of $V(t)$ in \eqref{equation:discrete_lyapunov_diff}, we have
  \begin{equation}
    \label{equation:discrete_delta_v_lb}
    \begin{aligned}
      \Delta V(t)
      &= \frac{\eta L(\wbf(t))}{r(t)}\inpr{Z\alphabf(t)}{P_{\ubf(t)}^\perp\ubf_*} - O\left(\frac{\eta^2\|\nabla L(\wbf(t))\|^2}{r(t)^2}\right) \\
      &\gtrsim -\frac{\eta L(\wbf(t))}{r(t)}\sqrt{2V(t)},
    \end{aligned}
  \end{equation}
  where the bound follows from $\|Z\alphabf\| \le 1$,
  $\|P_{\ubf}^\perp\ubf_*\|^2 \le 2V(t)$ (see the proof of Theorem~\ref{theorem:discrete_lyapunov_lower_bound}),
  $\|\nabla L(\wbf(t))\|\le L(\wbf(t))$ (the self-bounding property~\eqref{equation:self_bounding}),
  and the stepsize choice, collectively.
  Rearranging, for $V(t)>0$ we have the discrete inequality
  \begin{equation}
    \label{equation:discrete_v_sep_ineq}
    \frac{\Delta V(t)}{\sqrt{V(t)}} \gtrsim -\sqrt{2}\frac{\eta L(\wbf(t))}{r(t)}.
  \end{equation}
  By the Taylor expansion $\sqrt{V+\Delta V} = \sqrt{V} + \frac{\Delta V}{2\sqrt{V}} - O\left(\frac{(\Delta V)^2}{V^{3/2}}\right)$,
  we have
  \[
    \begin{aligned}
      \frac{\Delta V(s)}{\sqrt{V(s)}}
      &= 2\left[\sqrt{V(s+1)}-\sqrt{V(s)}\right] + O\left(\frac{(\Delta V(s))^2}{V(s)^{3/2}}\right) \\
      &\le 2\left[\sqrt{V(s+1)}-\sqrt{V(s)}\right] + O\left(\frac{\eta^2L(\wbf(s))^2}{\chi r(s)^2}\right),
      \quad \text{($\chi\defeq(\gamma(1-\gamma))^{3/2}$)}
    \end{aligned}
  \]
  where $|\Delta V(s)| \lesssim \eta L(\wbf(s))/r(s)$ (from \eqref{equation:discrete_delta_v_lb}) and $V(s) \ge \gamma(1-\gamma)$ (maintained before achieving $V(t)\le 1-\gamma$; see the discussion after Theorem~\ref{theorem:early_stage}) are used.
  Summing from $s=0$ to $s=t-1$ yields
  \[
    \begin{aligned}
      \sum_{s=0}^{t-1}\frac{\Delta V(s)}{\sqrt{V(s)}}
      &\le 2(\sqrt{V(t)} - \sqrt{V(0)}) + O\left(\frac{\eta^2}{\chi}\sum_{s=0}^{t-1}\frac{L(\wbf(s))^2}{r(s)^2}\right) \\
      &\le 2(\sqrt{V(t)} - \sqrt{V(0)}) + O\left(\frac{\eta}{\chi}\sum_{s=0}^{t-1}\frac{L(\wbf(s))}{r(s)}\right),
    \end{aligned}
  \]
  where we used the stepsize choice in the last inequality.
  Combining this with \eqref{equation:discrete_v_sep_ineq}, we have
  \begin{equation}
    \label{equation:discrete_v_telescope}
    \sqrt{V(t)}-\sqrt{V(0)}
    \gtrsim -\frac{\eta(1+O(\chi))}{\sqrt2}\sum_{s=0}^{t-1}\frac{L(\wbf(s))}{r(s)}.
  \end{equation}
  We now bound the sum using Lemma~\ref{lemma:discrete_l_r_bounds}.
  Defining
  \(
    T_0\defeq\frac{1}{L(0)\eta\gamma^2}\bigr[\exp\bigr(\frac{\rho+\log L(0)}{1+\gamma^{-2}}\bigl)-1\bigl]
  \), we have
  \begin{equation}
    \label{equation:discrete_r_lower_bound}
    r(t)\ge\begin{cases}
      \rho-\gamma^{-2}\log(1+L(0)\eta\gamma^2t) & \text{if $t\in[0,T_0]$,} \\
      \log(L(0)^{-1}+\eta\gamma^2t) & \text{if $t\ge T_0$.}
    \end{cases}
  \end{equation}
  At $t=T_0$, we have $r(T_0)\ge\frac{\rho\gamma^2-\log L(0)}{1+\gamma^2} \eqdef \underline{r} > 0$,
  where the strict positivity holds from the assumption $L(0)<e^{\rho\gamma^2}$.
  We split the sum into two intervals $s\in[0,T_0]\cup[T_0+1,t-1]$.
  For the first interval $s\in[0,T_0]$, we combine $r(s)\ge\underline{r}$
  with the upper bound of $L(s)$ from Lemma~\ref{lemma:discrete_l_r_bounds} to have
  \[
    \sum_{s=0}^{T_0}\frac{L(\wbf(s))}{r(s)}
    \le \frac{1}{\underline{r}}\left[L(0) + \int_0^{T_0}\frac{\rd s}{L(0)^{-1}+\eta\gamma^2s}\right]
    = \frac{1}{\underline{r}}\left[L(0) + \frac{\rho+\log L(0)}{\eta(1+\gamma^2)}\right].
  \]
  For the second interval $s\in[T_0+1,t-1]$, we combine \eqref{equation:discrete_r_lower_bound}
  with the upper bound of $L(s)$ to have
  \[
    \sum_{s=T_0+1}^{t-1}\frac{L(\wbf(s))}{r(s)}
    \le \int_{T_0}^{t}\frac{\rd s}{h(s)\log h(s)}
    = \frac{1}{\eta\gamma^2}\int_{h(T_0)}^{h(t)}\frac{\rd\varsigma}{\varsigma}
    = \frac{1}{\eta\gamma^2}\log\left(\frac{\log(L(0)^{-1}+\eta\gamma^2t)}{\frac{\rho\gamma^2-\log L(0)}{1+\gamma^{-2}}}\right),
  \]
  where we use the change of variable $\varsigma\defeq h(s)\defeq\log(L(0)^{-1}+\eta\gamma^2s)$.
  Thus, we combine these two intervals to have
  \[
    \sum_{s=0}^{t-1}\frac{L(\wbf(s))}{r(s)}
    \le \frac{(1+\gamma^2)L(0)}{\rho\gamma^2-\log L(0)} + \frac{\rho+\log L(0)}{\eta(\rho\gamma^2-\log L(0))} + \frac{1}{\eta\gamma^2}\log\left(\frac{\log(L(0)^{-1}+\eta\gamma^2t)}{\frac{\rho\gamma^2-\log L(0)}{1+\gamma^{-2}}}\right).
  \]
  By plugging this back into \eqref{equation:discrete_v_telescope}, we have
  \[
    \begin{aligned}
      &\sqrt{V(t)} \gtrsim \sqrt{V(0)} \\
      &- \frac{\eta(1+O(\chi))}{\sqrt2}\!\left[
        \frac{(1+\gamma^2)L(0)}{\rho\gamma^2-\log L(0)} + \frac{\rho+\log L(0)}{\eta(\rho\gamma^2-\log L(0))} + \frac{1}{\eta\gamma^2}\log\left(\frac{\log(L(0)^{-1}+\eta\gamma^2t)}{\frac{\rho\gamma^2-\log L(0)}{1+\gamma^{-2}}}\right)
      \right]\!.
    \end{aligned}
  \]
  By squaring both sides and relaxing the lower bound by $(a-b)^2\ge a^2-2ab$, we have
  \[
    V(t)
    \gtrsim V(0) - C_0 - \underbrace{\frac{\sqrt{2V(0)}(1+O(\chi))}{\gamma^2}}_{\eqdef C_1}\log\left(\log(L(0)^{-1}+\eta\gamma^2t)\right),
  \]
  where we can take a problem-dependent constant $C_0$.
  Lastly, to ensure $V(t)\ge\delta$, we need
  $\text{(lower bound)}\ge\delta$,
  which yields
  \[
    t <
      \frac{1}{\eta\gamma^2}\exp\left[\exp\left(\frac{V(0)-C_0-\delta}{C_1}\right)\right] - \frac{1}{L(0)\eta\gamma^2}.
  \]
\end{proof}

\begin{figure}[p]
  \centering
  \includegraphics[width=0.9\textwidth]{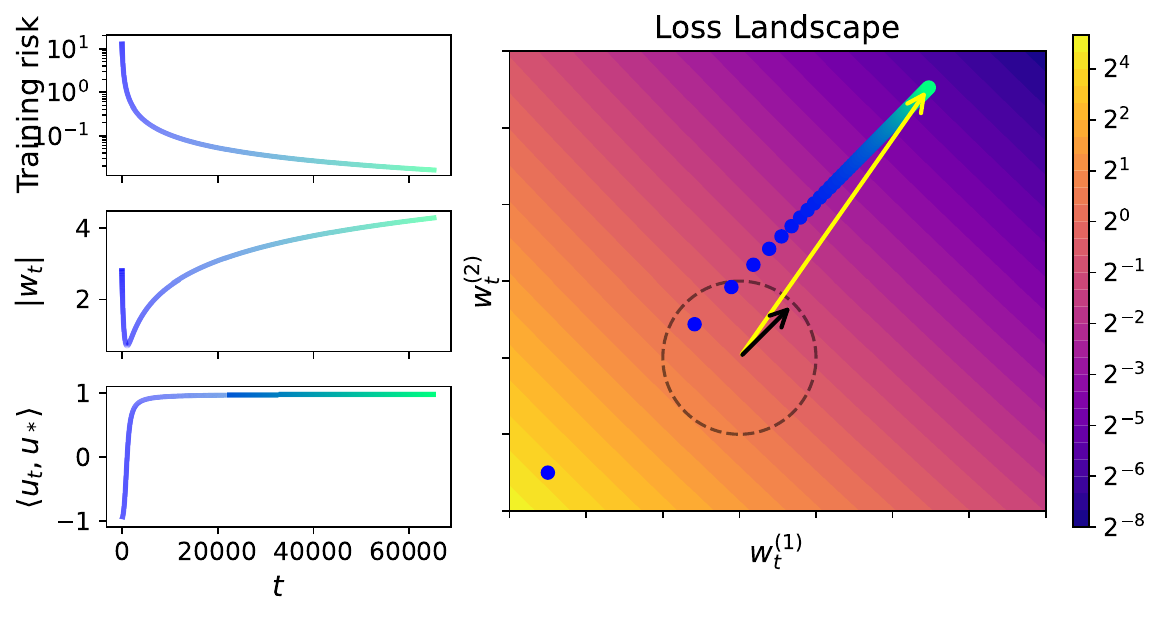} \\
  \includegraphics[width=0.6\textwidth]{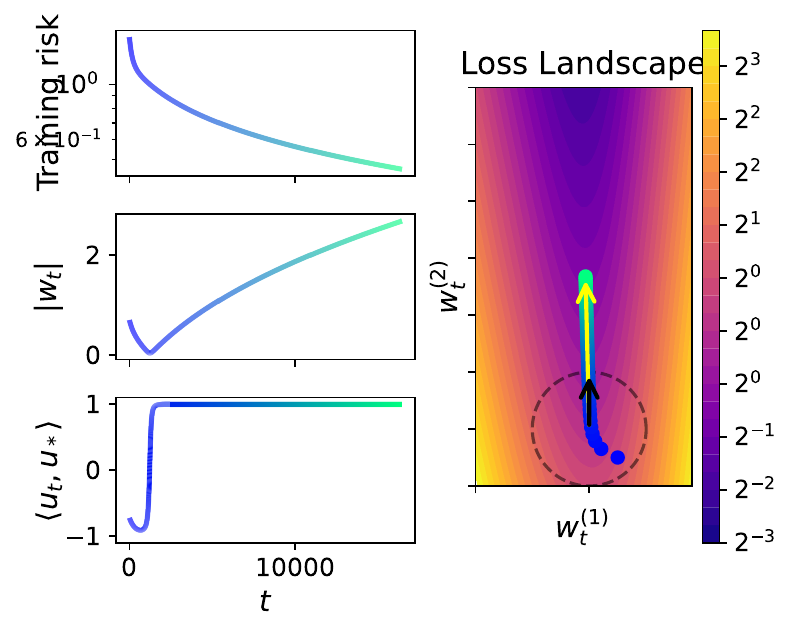}
  \caption{
    Simulation results with the unnormalized parameter trajectory $\wbf(t)$ depicted in the loss landscapes, unlike the tangential trajectory $\ubf(t)$ depicted in Figure~\ref{fig:simulation}.
    The simulation setup remains the same.
    We call the \emph{top} and \emph{bottom} setups, respectively.
  }
  \label{fig:simulation_full}
\end{figure}

\section{Simulation}
\label{section:simulation}

\paragraph{Expanded simulation results.}
We show simulation results under the same setup as in Figure~\ref{fig:simulation} but illustrate the unnormalized parameter trajectory $\wbf(t)$, instead of the tangential trajectory $\ubf(t)$.
The results are shown in Figure~\ref{fig:simulation_full}.
Again, the darker blue color in the (left) risk/radial/alignment curves and the (right) parameter trajectories indicates the earlier iteration $t$, and the lighter blue colors indicate the later iteration $t$.
In the loss landscapes, the parameter $\wbf(t)$ is plotted every $70$ and $50$ epochs for each dataset, respectively.
Evidently, the parameter trajectories quickly get away from initialization (corresponding to the ``scattered'' trajectory), and spend significantly more time when approaching the max-margin direction (corresponding to the smooth cyan trajectory).
The simulation details are as follows:
\begin{itemize}
  \item Top setup:
  \begin{itemize}
    \item Initialization $\wbf(0)=[-2.5, -1.5]^\top$
    \item Iterations $T=2^{16}$; stepsize $\eta=2^{-10}$
    \item Max-margin direction $\ubf_*=[1,1]^\top/\sqrt2$
    \item Dataset: generate $n=50$ points by sampling $\zbf\sim\Ncal(0, (0.3)^2I_2)$ and perturbing $\zbf$ with $+(2.5-\inpr{\ubf_*}{\zbf})\ubf_*$; margin $\gamma\sim0.95$
  \end{itemize}
  \item Bottom setup:
  \begin{itemize}
    \item Initialization $\wbf(0)=[-1.5, -0.5]^\top$
    \item Iterations $T=2^{14}$; stepsize $\eta=2^{-10}$
    \item Max-margin direction $\ubf_*=[0,1]^\top$
    \item Dataset: $\zbf_1=[1, 0.2]^\top$, $\zbf_2=[-2, 0.2]^\top$, $\zbf_3=[0.2, 1]^\top$, $\zbf_4=[2, 0.4]^\top$; margin $\gamma=0.2$
  \end{itemize}
\end{itemize}

\paragraph{Alignment time.}
We sweep margin $\gamma \in \{0.1, 0.2, 0.3, 0.4, 0.5\}$ over 2D datasets ($\wbf_* = [1,0]^\top$, $n=50$ points).
Two support vectors are placed at exactly margin $\gamma$, at $[\gamma, \varepsilon]^\top$ and $[-\gamma,-\varepsilon]^\top$ with $\varepsilon=0.01$, and the remaining points have margin strictly greater than $\gamma$; all points are normalized to $\|\xbf_i\| \le 1$.
The weights are initialized as $\wbf(0) \sim 0.01 \cdot \mathcal{N}(\mathbf{0}, \Ibf)$, independent of the data seed.
We use the exponential loss, stepsize $\eta = 2^{-10}$, and $T = 2^{14}$ iterations, and record the first iteration at which the alignment $\langle \wbf(t)/\|\wbf(t)\|, \wbf_* \rangle$ reaches each of the thresholds $\gamma$ (the standard margin threshold), $\sqrt{\gamma}$ (weak yet non-trivial margin threshold), and $0.999$ (the asymptotic convergence threshold).
Results are shown in Table~\ref{tab:alignment_time}, telling that (i) the weak alignment time (up to $\gamma$ and $\sqrt\gamma$) is much faster than the near-perfect alignment up to $0.999$;
(ii) the alignment up to $\sqrt\gamma$ does not significantly slow down with smaller $\gamma$ (hence more difficult datasets).

\begin{table}[p]
  \centering
  \caption{Alignment time (iterations) to reach each threshold, as a function of the dataset margin $\gamma$.}
  \vspace{2pt}
  \begin{tabular}{cc|ccc}
    \toprule
    Margin $\gamma$ & $\sqrt{\gamma}$ & Time to $\gamma$ & Time to $\sqrt{\gamma}$ & Time to $0.999$ \\
    \midrule
    $0.1$ & $0.316$ & $36$ & $39$ & $195$ \\
    $0.2$ & $0.447$ & $27$ & $30$ & $165$ \\
    $0.3$ & $0.548$ & $22$ & $25$ & $144$ \\
    $0.4$ & $0.632$ & $19$ & $22$ & $128$ \\
    $0.5$ & $0.707$ & $17$ & $20$ & $115$ \\
    \bottomrule
  \end{tabular}
  \label{tab:alignment_time}
\end{table}

\end{document}

%% file: preamble.tex
\usepackage{amsmath}
\usepackage{amssymb}
\usepackage{amsthm}
\usepackage{bm}
\usepackage{booktabs}
\usepackage[font=footnotesize,labelfont=bf]{caption}
\usepackage{dsfont}
\usepackage{enumitem}
\usepackage[T1]{fontenc}
\usepackage{mathtools}
\usepackage{mathrsfs}
\usepackage{multirow}
\usepackage{pifont}
\usepackage{rotating}
\usepackage[group-separator={,},group-minimum-digits={3}]{siunitx}
\usepackage{subcaption}
\usepackage[skins,breakable]{tcolorbox}
\tcbset{
  highlightbox/.style={
    colback=blue!5,colframe=blue!5,boxrule=0pt,breakable,
    left=4pt,right=4pt,top=4pt,bottom=4pt,before skip=4pt,after skip=4pt
  }
}
\usepackage{thmtools}
\usepackage{tikz}
\usetikzlibrary{arrows.meta,calc,external}
\usepackage{pgfplots}
\pgfplotsset{compat=1.18}
\usepackage{thm-restate}
\usepackage{wrapfig}

\setitemize{itemsep=-1pt,topsep=3pt}
\renewcommand{\tilde}{\widetilde}

\newcommand{\Ncal}{\mathcal{N}}

\newcommand{\Scal}{\mathcal{S}}

\newcommand{\Rbb}{\mathbb{R}}
\newcommand{\Sbb}{\mathbb{S}}

\newcommand{\Zbb}{\mathbb{Z}}

\newcommand{\Ibf}{\mathbf{I}}

\newcommand{\mbf}{\mathbf{m}}

\newcommand{\pbf}{\mathbf{p}}

\newcommand{\ubf}{\mathbf{u}}
\newcommand{\vbf}{\mathbf{v}}
\newcommand{\wbf}{\mathbf{w}}
\newcommand{\xbf}{\mathbf{x}}

\newcommand{\zbf}{\mathbf{z}}

\newcommand{\alphabf}{\bm{\alpha}}

\newcommand{\xibf}{\bm{\xi}}

\newcommand{\onebf}{\mathbf{1}}

\newcommand{\defeq}{\coloneqq}
\newcommand{\eqdef}{\eqqcolon}
\renewcommand{\subset}{\subseteq}

\DeclareMathOperator{\Span}{\mathrm{span}}

\newcommand{\inpr}[2]{\left\langle{#1},{#2}\right\rangle}

\newcommand{\rd}{\mathrm{d}}

\renewcommand{\epsilon}{\varepsilon}

\newtheorem{theorem}{Theorem}

\newtheorem{assumption}[theorem]{Assumption}
\newtheorem{lemma}[theorem]{Lemma}

\newtheorem{corollary}[theorem]{Corollary}
\newtheorem{remark}[theorem]{Remark}

\theoremstyle{remark}

%% file: figures/geometry_diagram.tex
\includegraphics{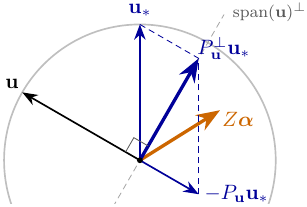}

%% file: figures/geometry_diagram2.tex
\includegraphics{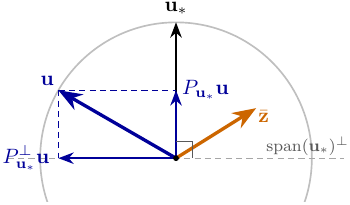}

%% file: figures/stages_diagram.tex
\includegraphics{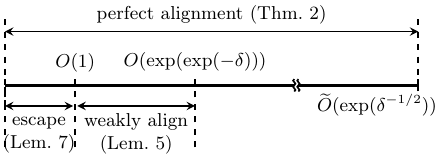}

%% file: figures/kappa_plot.tex
\includegraphics{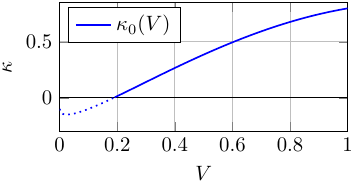}